\documentclass{article}
\pdfoutput=1
     \PassOptionsToPackage{numbers, compress}{natbib}

\usepackage[final]{neurips_2022}

\newcommand{\ignore}[1]{}
\usepackage{dsfont}

\usepackage[T1]{fontenc}
\usepackage[utf8]{inputenc}
\usepackage{amsmath,amssymb,amsfonts,mathrsfs,bm}
\usepackage{mathtools}
\usepackage{amsthm}
\usepackage{scalerel}
\usepackage{nicefrac}
\usepackage{microtype} 
\usepackage[shortlabels]{enumitem}
\usepackage{graphicx}
\usepackage{epstopdf}
\DeclareGraphicsExtensions{.eps,.png,.jpg,.pdf}

\usepackage{url}
\usepackage{colortbl}
\usepackage{booktabs}
\usepackage{multirow}
\usepackage[table,dvipsnames]{xcolor}
\usepackage[normalem]{ulem}
\usepackage{xparse}
\usepackage{calc}
\usepackage{etoolbox}

\makeatletter
\@ifpackageloaded{natbib}{
	\relax
}{
	\usepackage{cite}
}
\makeatother

\usepackage{array}
\newcolumntype{L}[1]{>{\raggedright\let\newline\\\arraybackslash\hspace{0pt}}m{#1}}
\newcolumntype{C}[1]{>{\centering\let\newline\\\arraybackslash\hspace{0pt}}m{#1}}
\newcolumntype{R}[1]{>{\raggedleft\let\newline\\\arraybackslash\hspace{0pt}}m{#1}}

\makeatletter
\let\MYcaption\@makecaption
\makeatother
\usepackage[font=footnotesize]{subcaption}
\makeatletter
\let\@makecaption\MYcaption
\makeatother

\usepackage{glossaries}
\makeatletter
\sfcode`\.1006

\let\oldgls\gls
\let\oldglspl\glspl

\newcommand\fussy@ifnextchar[3]{%
	\let\reserved@d=#1%
	\def\reserved@a{#2}%
	\def\reserved@b{#3}%
	\futurelet\@let@token\fussy@ifnch}
\def\fussy@ifnch{%
	\ifx\@let@token\reserved@d
		\let\reserved@c\reserved@a
	\else
		\let\reserved@c\reserved@b
	\fi
	\reserved@c}

\renewcommand{\gls}[1]{%
\oldgls{#1}\fussy@ifnextchar.{\@checkperiod}{\@}}
\renewcommand{\glspl}[1]{%
\oldglspl{#1}\fussy@ifnextchar.{\@checkperiod}{\@}}

\newcommand{\@checkperiod}[1]{%
	\ifnum\sfcode`\.=\spacefactor\else#1\fi
}

\robustify{\gls}
\robustify{\glspl}
\makeatother

\newacronym{wrt}{w.r.t.}{with respect to}
\newacronym{RHS}{R.H.S.}{right-hand side}
\newacronym{LHS}{L.H.S.}{left-hand side}
\newacronym{iid}{i.i.d.}{independent and identically distributed}

\usepackage{float}

\ifx\notloadhyperref\undefined
	\ifx\loadbibentry\undefined
		\usepackage[hidelinks,hypertexnames=false]{hyperref}
	\else
		\usepackage{bibentry}
		\makeatletter\let\saved@bibitem\@bibitem\makeatother
		\usepackage[hidelinks,hypertexnames=false]{hyperref}
		\makeatletter\let\@bibitem\saved@bibitem\makeatother
	\fi
\else
	\ifx\loadbibentry\undefined
		\relax
	\else
		\usepackage{bibentry}
	\fi
\fi

\usepackage[capitalize]{cleveref}
\crefname{equation}{}{}
\Crefname{equation}{}{}
\crefname{claim}{claim}{claims}
\crefname{step}{step}{steps}
\crefname{line}{line}{lines}
\crefname{condition}{condition}{conditions}
\crefname{dmath}{}{}
\crefname{dseries}{}{}
\crefname{dgroup}{}{}

\crefname{Problem}{Problem}{Problems}
\crefformat{Problem}{Problem~(#2#1#3)}
\crefrangeformat{Problem}{Problems~(#3#1#4) to~(#5#2#6)}

\crefname{Theorem}{Theorem}{Theorems}
\crefname{Corollary}{Corollary}{Corollaries}
\crefname{Proposition}{Proposition}{Propositions}
\crefname{Lemma}{Lemma}{Lemmas}
\crefname{Definition}{Definition}{Definitions}
\crefname{Example}{Example}{Examples}
\crefname{Assumption}{Assumption}{Assumptions}
\crefname{Remark}{Remark}{Remarks}
\crefname{Rem}{Remark}{Remarks}
\crefname{remarks}{Remarks}{Remarks}
\crefname{Appendix}{Appendix}{Appendices}
\crefname{Supplement}{Supplement}{Supplements}
\crefname{Exercise}{Exercise}{Exercises}
\crefname{Theorem_A}{Theorem}{Theorems}
\crefname{Corollary_A}{Corollary}{Corollaries}
\crefname{Proposition_A}{Proposition}{Propositions}
\crefname{Lemma_A}{Lemma}{Lemmas}
\crefname{Definition_A}{Definition}{Definitions}

\usepackage{crossreftools}
\ifx\notloadhyperref\undefined
\else
	\relax
\fi

\usepackage{algorithm}
\usepackage{algpseudocode}

\ifx\loadbreqn\undefined
	\relax
\else
	\usepackage{breqn}
\fi

\makeatletter
\def\cleartheorem#1{%
    \expandafter\let\csname#1\endcsname\relax
    \expandafter\let\csname c@#1\endcsname\relax
}
\def\clearthms#1{ \@for\tname:=#1\do{\cleartheorem\tname} }
\makeatother

\ifx\renewtheorem\undefined
	\ifx\useTheoremCounter\undefined
		\newtheorem{Theorem}{Theorem}
		\newtheorem{Corollary}{Corollary}
		\newtheorem{Proposition}{Proposition}
		\newtheorem{Lemma}{Lemma}
	\else
		\newtheorem{Theorem}{Theorem}
		
		\newtheorem{Proposition}[Theorem]{Proposition}
	\fi

	\newtheorem{Definition}{Definition}

	\newtheorem{Remark_A}{Remark}[section]

\fi

\theoremstyle{remark}

\theoremstyle{plain}

\newcommand{\qednew}{\nobreak \ifvmode \relax \else
		\ifdim\lastskip<1.5em \hskip-\lastskip
			\hskip1.5em plus0em minus0.5em \fi \nobreak
		\vrule height0.75em width0.5em depth0.25em\fi}



\newcommand{\nn}{\nonumber\\ }

\NewDocumentCommand{\movedownsub}{e{^_}}{%
	\IfNoValueTF{#1}{%
		\IfNoValueF{#2}{^{}}
	}{%
		^{#1}
	}%
	\IfNoValueF{#2}{_{#2}}
}

\let\latexchi\chi
\RenewDocumentCommand{\chi}{}{\latexchi\movedownsub}

\newcommand{\Real}{\mathbb{R}}

\newcommand{\calE}{\mathcal{E}}

\newcommand{\calG}{\mathcal{G}}
\newcommand{\calH}{\mathcal{H}}

\newcommand{\calL}{\mathcal{L}}

\newcommand{\calV}{\mathcal{V}}

\newcommand{\bA}{\mathbf{A}}

\newcommand{\bB}{\mathbf{B}}

\newcommand{\bD}{\mathbf{D}}

\newcommand{\bE}{\mathbf{E}}

\newcommand{\bI}{\mathbf{I}}

\newcommand{\bQ}{\mathbf{Q}}

\newcommand{\bW}{\mathbf{W}}

\newcommand{\bz}{\mathbf{z}}
\newcommand{\bZ}{\mathbf{Z}}

\DeclareSymbolFont{bsfletters}{OT1}{cmss}{bx}{n}
\DeclareSymbolFont{ssfletters}{OT1}{cmss}{m}{n}
\DeclareMathSymbol{\bsfGamma}{0}{bsfletters}{'000}
\DeclareMathSymbol{\ssfGamma}{0}{ssfletters}{'000}
\DeclareMathSymbol{\bsfDelta}{0}{bsfletters}{'001}
\DeclareMathSymbol{\ssfDelta}{0}{ssfletters}{'001}
\DeclareMathSymbol{\bsfTheta}{0}{bsfletters}{'002}
\DeclareMathSymbol{\ssfTheta}{0}{ssfletters}{'002}
\DeclareMathSymbol{\bsfLambda}{0}{bsfletters}{'003}
\DeclareMathSymbol{\ssfLambda}{0}{ssfletters}{'003}
\DeclareMathSymbol{\bsfXi}{0}{bsfletters}{'004}
\DeclareMathSymbol{\ssfXi}{0}{ssfletters}{'004}
\DeclareMathSymbol{\bsfPi}{0}{bsfletters}{'005}
\DeclareMathSymbol{\ssfPi}{0}{ssfletters}{'005}
\DeclareMathSymbol{\bsfSigma}{0}{bsfletters}{'006}
\DeclareMathSymbol{\ssfSigma}{0}{ssfletters}{'006}
\DeclareMathSymbol{\bsfUpsilon}{0}{bsfletters}{'007}
\DeclareMathSymbol{\ssfUpsilon}{0}{ssfletters}{'007}
\DeclareMathSymbol{\bsfPhi}{0}{bsfletters}{'010}
\DeclareMathSymbol{\ssfPhi}{0}{ssfletters}{'010}
\DeclareMathSymbol{\bsfPsi}{0}{bsfletters}{'011}
\DeclareMathSymbol{\ssfPsi}{0}{ssfletters}{'011}
\DeclareMathSymbol{\bsfOmega}{0}{bsfletters}{'012}
\DeclareMathSymbol{\ssfOmega}{0}{ssfletters}{'012}

\newcommand{\bPsi}{\bm{\Psi}}

\makeatletter
\newcommand*\rel@kern[1]{\kern#1\dimexpr\macc@kerna}
\newcommand*\widebar[1]{%
  \begingroup
  \def\mathaccent##1##2{%
    \rel@kern{0.8}%
    \overline{\rel@kern{-0.8}\macc@nucleus\rel@kern{0.2}}%
    \rel@kern{-0.2}%
  }%
  \macc@depth\@ne
  \let\math@bgroup\@empty \let\math@egroup\macc@set@skewchar
  \mathsurround\z@ \frozen@everymath{\mathgroup\macc@group\relax}%
  \macc@set@skewchar\relax
  \let\mathaccentV\macc@nested@a
  \macc@nested@a\relax111{#1}%
  \endgroup
}
\makeatother

\newcommand{\ifbcdot}[1]{\ifblank{#1}{\cdot}{#1}}

\DeclarePairedDelimiterX\abs[1]{\lvert}{\rvert}{\ifbcdot{#1}}
\DeclarePairedDelimiterX\parens[1]{(}{)}{\ifbcdot{#1}}
\DeclarePairedDelimiterX\brk[1]{[}{]}{\ifbcdot{#1}}
\DeclarePairedDelimiterX\braces[1]{\{}{\}}{\ifbcdot{#1}}
\DeclarePairedDelimiterX\angles[1]{\langle}{\rangle}{\ifblank{#1}{\cdot,\cdot}{#1}}
\DeclarePairedDelimiterX\ip[2]{\langle}{\rangle}{\ifbcdot{#1},\ifbcdot{#2}}
\DeclarePairedDelimiterX\norm[1]{\lVert}{\rVert}{\ifbcdot{#1}}
\DeclarePairedDelimiterX\ceil[1]{\lceil}{\rceil}{\ifbcdot{#1}}
\DeclarePairedDelimiterX\floor[1]{\lfloor}{\rfloor}{\ifbcdot{#1}}

\DeclareFontFamily{U}{matha}{\hyphenchar\font45}
\DeclareFontShape{U}{matha}{m}{n}{
      <5> <6> <7> <8> <9> <10> gen * matha
      <10.95> matha10 <12> <14.4> <17.28> <20.74> <24.88> matha12
      }{}
\DeclareSymbolFont{matha}{U}{matha}{m}{n}
\DeclareFontSubstitution{U}{matha}{m}{n}

\DeclareFontFamily{U}{mathx}{\hyphenchar\font45}
\DeclareFontShape{U}{mathx}{m}{n}{
      <5> <6> <7> <8> <9> <10>
      <10.95> <12> <14.4> <17.28> <20.74> <24.88>
      mathx10
      }{}
\DeclareSymbolFont{mathx}{U}{mathx}{m}{n}
\DeclareFontSubstitution{U}{mathx}{m}{n}

\DeclareMathDelimiter{\vvvert}{0}{matha}{"7E}{mathx}{"17}
\DeclarePairedDelimiterX\vertiii[1]{\vvvert}{\vvvert}{\ifbcdot{#1}}

\DeclarePairedDelimiterXPP\trace[1]{\operatorname{Tr}}{(}{)}{}{\ifbcdot{#1}} 
\DeclarePairedDelimiterXPP\col[1]{\operatorname{col}}{\{}{\}}{}{\ifbcdot{#1}} 
\DeclarePairedDelimiterXPP\row[1]{\operatorname{row}}{\{}{\}}{}{\ifbcdot{#1}} 
\DeclarePairedDelimiterXPP\erf[1]{\operatorname{erf}}{(}{)}{}{\ifbcdot{#1}}
\DeclarePairedDelimiterXPP\erfc[1]{\operatorname{erfc}}{(}{)}{}{\ifbcdot{#1}}
\DeclarePairedDelimiterXPP\KLD[2]{D}{(}{)}{}{\ifbcdot{#1}\, \delimsize\|\, \ifbcdot{#2}} 
\DeclarePairedDelimiterXPP\op[2]{\operatorname{#1}}{(}{)}{}{#2} 

\newcommand{\T}{^{\intercal}}
\newcommand{\ud}{\,\mathrm{d}} 

\DeclarePairedDelimiterXPP\indicate[1]{{\bf 1}}{\{}{\}}{}{\ifbcdot{#1}}

\providecommand\given{}

\DeclarePairedDelimiterX\Set[2]\{\}{%
\renewcommand\given{\SetSymbol[\delimsize]{#1}}
#2
}
\DeclarePairedDelimiterX\Setc[1]\{\}{%
\renewcommand\given{\SetSymbol{:}}
#1
}

\NewDocumentCommand\set{s o m}{%
	\IfBooleanTF#1%
	{\IfValueTF{#2}{\Set*{#2}{#3}}{\Setc*{#3}}}%
	{\IfValueTF{#2}{\Set{#2}{#3}}{\Setc{#3}}}%
}

\NewDocumentCommand{\evalat}{ s O{\big} m e{_^} }{%
\IfBooleanTF{#1}%
{\left. #3 \right|}{#3#2|}%
\IfValueT{#4}{_{#4}}%
\IfValueT{#5}{^{#5}}%
}

\providecommand\given{}
\DeclarePairedDelimiterXPP\cprob[1]{}(){}{
\renewcommand\given{\nonscript\,\delimsize\vert\allowbreak\nonscript\,\mathopen{}}%
#1%
}
\DeclarePairedDelimiterXPP\cexp[1]{}[]{}{
\renewcommand\given{\nonscript\,\delimsize\vert\allowbreak\nonscript\,\mathopen{}}%
#1%
}

\DeclareDocumentCommand \P { s e{_^} d() g } {%
	\mathbb{P}%
	\IfBooleanTF{#1}%
		{
			\IfValueT{#2}{_{#2}}%
			\IfValueT{#3}{^{#3}}%
			\IfValueTF{#5}{\cprob{#4 \given #5}}{\IfValueT{#4}{\cprob{#4}}}%
		}%
		{
			\IfValueT{#2}{_{#2}}%
			\IfValueT{#3}{^{#3}}%
			\IfValueTF{#5}{\cprob*{#4 \given #5}}{\IfValueT{#4}{\cprob*{#4}}}%
		}%
}

\DeclareDocumentCommand \E { s e{_^} o g } {%
	\mathbb{E}%
	\IfBooleanTF{#1}%
		{
			\IfValueT{#2}{_{#2}}%
			\IfValueT{#3}{^{#3}}%
			\IfValueTF{#5}{\cexp{#4 \given #5}}{\IfValueT{#4}{\cexp{#4}}}%
		}%
		{
			\IfValueT{#2}{_{#2}}%
			\IfValueT{#3}{^{#3}}%
			\IfValueTF{#5}{\cexp*{#4 \given #5}}{\IfValueT{#4}{\cexp*{#4}}}%
		}%
}

\ExplSyntaxOn
\NewDocumentCommand \dist {m o o} {%
\mathrm{#1}\left(%
	\IfValueT{#3}{%
		\tl_if_blank:nTF{ #3 }{\cdot\, \middle|\, }{#3\, \middle|\, }%
	}
	\IfValueT{#2}{#2}%
\right)%
}
\ExplSyntaxOff

\NewDocumentCommand {\cbrace} {t+ D[]{black} D(){\widthof{#5}} m m } {%
	\begingroup%
		\color{#2}
		\IfBooleanTF{#1}{%
			\overbrace{#4}^%
		}{
			\underbrace{#4}_%
		}%
		{\parbox[c]{#3}{\centering\footnotesize{#5}}}%
	\endgroup%
}

\let\oldforall\forall
\renewcommand{\forall}{\oldforall \, }

\let\oldexist\exists
\renewcommand{\exists}{\oldexist \, }

\graphicspath{{./Figures/}{./figures/}}
\DeclareDocumentCommand{\includeCroppedPdf}{ o O{./Figures/} m }{
	\IfFileExists{#2#3-crop.pdf}{}{%
		\immediate\write18{pdfcrop #2#3.pdf #2#3-crop.pdf}}%
	\includegraphics[#1]{#2#3-crop.pdf}
}

\makeatletter
\newcommand*{\addFileDependency}[1]{
  \typeout{(#1)}
  \@addtofilelist{#1}
  \IfFileExists{#1}{}{\typeout{No file #1.}}
}
\makeatother

\definecolor{gray90}{gray}{0.9}

\ifx\nohighlights\undefined
	\newcommand{\red}[1]{{\color{red} #1}}
	\newcommand{\blue}[1]{{{\color{blue} #1}}}

	\newcommand{\msout}[1]{\text{\color{green} \sout{\ensuremath{#1}}}}
	\newcommand{\del}[1]{{\color{green}\ifmmode \msout{#1}\else\sout{#1}\fi}}
\else
	\newcommand{\red}[1]{#1}
	\newcommand{\blue}[1]{#1}

	\newcommand{\msout}[1]{#1}
	\newcommand{\del}[1]{#1}
\fi

\newcommand{\hhide}[1]{}

\ifx\diagnoselabel\undefined
	\relax
\else
	\makeatletter
	\def\@testdef #1#2#3{%
		\def\reserved@a{#3}\expandafter \ifx \csname #1@#2\endcsname
			\reserved@a  \else
			\typeout{^^Jlabel #2 changed:^^J%
				\meaning\reserved@a^^J%
				\expandafter\meaning\csname #1@#2\endcsname^^J}%
			\@tempswatrue \fi}
	\makeatother
\fi

\usepackage{tikz}
\usetikzlibrary{shapes.geometric, arrows}
\title{On the Robustness of Graph Neural Diffusion to Topology Perturbations}
\author{%
 Yang~Song\thanks{Equal contribution.}\\
  Nanyang Technological University\\
     C3 AI\\
   \texttt{yang.song@c3.ai} \\
      \And
  Qiyu~Kang\footnotemark[1]\\
  Nanyang Technological University\\
   \texttt{kang0080@e.ntu.edu.sg} \\
   \And 
   Sijie~Wang\footnotemark[1]\\
  Nanyang Technological University\\
   \texttt{wang1679@e.ntu.edu.sg} \\
   \And
      Kai~Zhao\footnotemark[1]\\
  Nanyang Technological University\\
    \texttt{kai.zhao@ntu.edu.sg} \\
   \And
   Wee~Peng~Tay\\
  Nanyang Technological University\\
   \texttt{wptay@ntu.edu.sg} \\
}

\begin{document}

\maketitle

\begin{abstract}

Neural diffusion on graphs is a novel class of graph neural networks that has attracted increasing attention recently. The capability of graph neural partial differential equations (PDEs) in addressing common hurdles of graph neural networks (GNNs), such as the problems of over-smoothing and bottlenecks, has been investigated but not their robustness to adversarial attacks. In this work, we explore the robustness properties of graph neural PDEs. We empirically demonstrate that graph neural PDEs are intrinsically more robust against topology perturbation as compared to other GNNs. We provide insights into this phenomenon by exploiting the stability of the heat semigroup under graph topology perturbations. We discuss various graph diffusion operators and relate them to existing graph neural PDEs. Furthermore, we propose a general graph neural PDE framework based on which a new class of robust GNNs can be defined. We verify that the new model achieves comparable state-of-the-art performance on several benchmark datasets.
\end{abstract}

\section{Introduction}
\label{sect:intro}
Deep learning on graphs and Graph Neural Networks (GNNs), in particular, have achieved remarkable success in a variety of application areas such as those related to chemistry (molecules), finance (trading networks) and the social media (the Facebook friend network) \cite{yueBio2019,AshoorNC2020,kipf2017semi,ZhangTKDE2022,WuTNNLS2021}. GNNs have been applied to various tasks including node classification \cite{kipf2017semi}, link prediction \cite{kipf2016vgae}, and recommender systems \cite{YingKDD2018}. The key to the success of GNNs is the neural message passing scheme \cite{GilmerICML2017} where messages are propagated along edges and optimized towards a downstream task. 

While the aggregation of neighboring nodes' information is a powerful principle of representation learning, the way that GNNs exchange the information between nodes makes them vulnerable to adversarial attacks \cite{zugnerKDD2018}. Adversaries can perturb a graph's topology by adding or removing edges \cite{Chen2018FastGA,WaniekNHB2018,Du2017TopologyAG} or by injecting malicious nodes to the original graph \cite{Wang2020ScalableAO,speit_attack,zouKDD2021}.
Another common attack is to perturb node attributes \cite{zugnerKDD2018,zugner_adversarial_2019,maKDD2021,sunWWW2020}. 
Our paper will mainly tackle graph topology perturbation. 
Adversaries who can inject nodes to the original graph while not modifying the original graph directly are called injection attacks \cite{Wang2020ScalableAO,speit_attack,zouKDD2021}. Adversaries who can directly modify the original graph including edges and node features are called modification attacks \cite{Chen2018FastGA,WaniekNHB2018,Du2017TopologyAG,zugnerKDD2018,zugner_adversarial_2019,maKDD2021,sunWWW2020}.
 To defend against adversarial attacks, several robust GNN models have been proposed. Examples include RobustGCN \cite{zhuKDD2019}, GRAND \cite{FengNeurips2020}, and ProGNN \cite{JinKDD2020}. In addition, pre-processing based defenders include GNN-SVD \cite{EntezariWSDM2020} and GNNGuard \cite{zhangNeurips2020}. 

Recent studies \cite{yan2019robustness,haber2017stable,liu2020does,kang2021Neurips} have applied neural Ordinary Differential Equations (ODEs) \cite{chen2018neural} to defend against adversarial attacks. Some works like \cite{yan2019robustness,kang2021Neurips} have revealed interesting intrinsic properties of ODEs that make them more stable than conventional convolutional neural networks (CNNs). Neural Partial Differential Equations (PDEs) have also been applied to graph-structured data \cite{chamberlain2021grand, chamberlain2021blend}. Some papers like \cite{chamberlain2021grand} (GRAND) and \cite{chamberlain2021blend} (BLEND) approach deep learning on graphs as a continuous diffusion process and treat GNNs as spatial discretizations of an underlying PDE. However, robustness to adversarial attack has not been studied on such graph neural PDEs. In this work, we demonstrate that graph neural PDEs are intrinsically more robust against adversarial topological perturbations compared to other GNNs. We investigate the heat diffusion on a general Riemannian manifold and show that the diffusion process is essentially stable under small perturbations of the manifold metric. In doing so, we provide insights into the underlining reasons why graph neural PDEs are stable under graph topological perturbations. Such insights indicate that further improvements to the design of graph neural PDEs are possible. We develop several such improved models under a general graph neural PDE framework and show that these models are also robust to node attribute perturbations.

\textbf{Main contributions}. In this paper, our objective is to develop a general diffusion framework on graphs and study the robustness properties of the induced graph neural PDEs. Our main contributions are summarized as follows:
\begin{itemize}
    \item We review the notion of heat diffusion on Riemannian manifolds and the stability of its semigroup. We present analogous concepts of gradient, divergence and Laplacian operators for heat diffusion on a graph. 
    \item We generalize heat flow to more general flow schemes, including mean curvature flow and Beltrami flow, which are able to preserve inter-class edges in diffusion. A novel class of graph neural PDEs is thereby induced. 
    \item  We show that the proposed graph neural PDEs are intrinsically robust to graph topology perturbations. We verify that the new model achieves
comparable state-of-the-art performance on several benchmark datasets.
\end{itemize}

The rest of this paper is organized as follows. In \cref{sect:preliminaries}, we start with preliminaries on continuous diffusion over a Riemannian manifold and the discrete graph. In \cref{sect:neural_gode}, we present our main results on the stability properties of heat flow on graphs and generalize heat flow to edge-preserving flows from which a new class of graph neural PDEs is proposed. Our proposed model architecture is detailed in \cref{sect:model}. We summarize experimental results in \cref{sect:exper} and conclude the paper in \cref{sect:conc}. The proofs for all theoretical results in this paper are given in the supplementary material, where more experiments are also presented.

\section{Preliminaries}
\label{sect:preliminaries}
Similar to \cite{chamberlain2021grand, chamberlain2021blend}, we consider a graph as a discretization of a Riemannian manifold. We introduce concepts and notations for flows diffused over a general manifold and a discrete graph. Readers are referred to \cite{grigoryan2009heat} for more details. The stability of the heat kernel and heat semigroup of the heat diffusion equation under perturbations of the manifold metric is considered.  This discussion sheds light on the diffusion stability on graphs under perturbation of the graph topology in the next section.

\subsection{Heat Equation and Solution Stability on Manifold}\label{sec:heat_manifold}

A Riemannian manifold $(M,g)$ is a smooth manifold $M$ endowed with a Riemannian metric $g$, where the norms of tangent vectors and the angles between them are defined by an inner product. It is the natural generalization of the Euclidean space and correspondingly the divergence, gradient and Laplace operators are well-defined, analogous to their corresponding concepts in Euclidean space. Let $C_0^{\infty}(M)$ denote the space of smooth functions on $M$ with compact support and $\circ$ the composition operation.

\begin{Definition}\label{def:div_gradient_manifold}
The \emph{Laplace operator} $\Delta : C_0^{\infty}(M)\mapsto C_0^{\infty}(M)$ on a $d$-dimensional Riemannian manifold $(M, g)$ is defined as
\begin{align*}
   \Delta=\operatorname{div} \circ \nabla,
\end{align*}
where $\operatorname{div}$ is the \emph{divergence} defined for the $C_0^{\infty}$ vector fields on $M$ and  $\nabla$ is the \emph{gradient} operator.
In a local chart $U$ with coordinates $(x^1,x^2,\dots,x^n)$, for function $f$ and vector field $\psi$, we have
\begin{align}
    (\nabla f)^{i}=\sum_{j=1}^{d}g^{i j} \frac{\partial f}{\partial x^{j}}, \quad \operatorname{div} \psi=\sum_{j=1}^{d}\frac{1}{\sqrt{\operatorname{det} g}} \frac{\partial}{\partial x^{j}}\left(\sqrt{\operatorname{det} g} \psi^{j}\right), \label{eq:div_gradient_manifold}
\end{align}
where $\psi^{j}$ is the $j$-th component of $\psi$, $g^{ij}$ the components of the inverse metric of $g$, and $\operatorname{det}$ is the determinant operator of a matrix. We also have
\begin{align}
    \Delta=\sum_{i,j}^{d}\frac{1}{\sqrt{\operatorname{det} g}} \frac{\partial}{\partial x^{i}}\left(\sqrt{\operatorname{det}g}  g^{i j} \frac{\partial}{\partial x^{j}}\right).
\end{align}
\end{Definition}

 The classical Cauchy problem associated with the heat diffusion equation is to find a function $\varphi(t, x) \in C_0^{\infty}\left(\mathbb{R}^{+} \times M\right)$ such that
\begin{align}
\left\{\begin{array}{l}
\frac{\partial \varphi}{\partial t}=\Delta \varphi,\ t>0, \\
\left.\varphi\right|_{t=0}=f,
\end{array}\right. \label{eq:heat_cauchy}
\end{align}
where we only consider $f\in C_0^{\infty}$ and $\mathbb{R}^{+}$ is the space of positive real numbers. 

We can extend the Laplace operator to the generalized \emph{Dirichlet Laplace operator}, denoted as $\calL=-\Delta$ and defined on a larger domain\footnote{$\calL$ is defined on the larger Sobolev space $W^2_0(M)$ \cite{{grigoryan2009heat}}, and is the (negative) extension of $\Delta$. We have $\calL=-\Delta$ on $C^{\infty}_0$.} \cite[section 4.2]{grigoryan2009heat}. This operator is self-adjoint and non-negative definite on $L^2(M)$, and it can be shown that the above problem \cref{eq:heat_cauchy} is solved by means of the following family $\{P_{t}\}_{t \geq 0}$ of operators:
\begin{align}
P_{t} := e^{-t \calL}=\int_{\mathrm{spec}\calL} e^{-t \lambda} \ud E_{\lambda}
=\int_{0}^{\infty} e^{-t \lambda} \ud E_{\lambda},
\end{align}
where $\left\{E_{\lambda}\right\}_{\lambda\in\mathrm{spec}\calL}$ in $L^2(M)$ is the unique \emph{spectral resolution} of the {Dirichlet Laplace operator} $\calL$ and $\mathrm{spec}\calL$ is the \emph{spectrum set} of $\calL$. The family $\left\{P_{t}\right\}_{t \geq 0}$ is called the \emph{heat semigroup} associated with $\Delta$, and the solution is given by $P_tf$. It is also well-known that the solution of the above problem has an integration form via the (minimal) heat kernel function $k_t(x,y):\mathbb{R}^{+} \times M \times M \mapsto \mathbb{R}$ \cite[Theorem 7.13]{grigoryan2009heat}:
\begin{align*}
    P_{t} f(x)=\int_{M} k_{t}(x, y) f(y) \ud\mu(y),
\end{align*}
where  $\ud\mu$ is the Riemannian measure \cite[Section 3.4]{grigoryan2009heat} on $M$.

For a Riemannian manifold $M$, different metrics $g$ give rise to different structures. Let $\widetilde{\Delta}$ be the Laplacian associated with another metric $\tilde{g}$ on $M$ such that, for some $\alpha \geq 1$,
\begin{align*}
   \alpha^{-1} \tilde{g} \leq g \leq \alpha \tilde{g}.
\end{align*}
The new  $\widetilde{\Delta}$ is said to be \emph{quasi-isometric} \cite{saloff1992uniformly} to $g$, and can be viewed as a \emph{perturbation} of the metric $g$. $\widetilde{\Delta}$ can be also viewed as a uniformly elliptic operator \cite{saloff1992uniformly} \gls{wrt} to $g$. The stability of the heat semigroup and the heat kernel under perturbations of the Laplace operator (i.e., changes of $g$) is well studied in
\cite{chen1998stability}. In the special case of uniformly elliptic operators on Euclidean manifold $\Real^d$,  
the uniformly elliptic operator can be writen as $\operatorname{div} (A \nabla)$ with symmetric measurable coefficients $A(x)=\left(a_{i j}(x)\right)$ such that $\sum_{i,j}a_{i j}(x) \xi_{i} \xi_{j} \geq \lambda_{0}|\xi|^{2}$
for all $\xi \in \Real^{d}$ and for some constant $\lambda>0$ independent of $x$.
The reference \cite{chen1998stability} shows the following theorem:

\begin{Theorem}\label{thm:bound1}
On the Euclidean manifold $\Real^n$, let $\{P_{t}\}_{t \ge 0}$ and $\{\widetilde{P}_{t}\}_{t \ge 0}$ be two diffusion semigroups on $\mathbb{R}^{d}$ $(d \ge 2)$ associated with uniformly elliptic operators $\operatorname{div} (A \nabla)$ and $\operatorname{div}(\tilde{A} \nabla)$ with symmetric measurable coefficients $A(x)=\left(a_{i j}(x)\right)$ and $\tilde{A}(x)=\left(\tilde{a}_{i j}(x)\right)$, respectively. The corresponding heat kernels are denoted by $p_{t}(x, y)$ and $\tilde{p}_{t}(x, y)$.
We then have that 
\begin{enumerate}[i.]
\item There is a bounded, piecewise continuous function $F_1(t, z)$ on $\mathbb{R}^{+}\times\mathbb{R}^{+}$ with $\lim _{z \rightarrow 0} F_1(t, z)=0$ for each $t>0$ and a constant $c>0$, both of which depend only on $d$ and $\alpha$, such that
\begin{align*}
\left|p_{t}(x, y)-\tilde{p}_{t}(x, y)\right| \le t^{-d / 2} \exp \left(-\frac{\|x-y\|_2^{2}}{c t}\right) F_1\left(t,\|A-\tilde{A}\|_{L_{loc}^{2}}\right),
\end{align*}
for any $(t, x, y) \in \mathbb{R}^{+} \times \mathbb{R}^d \times \mathbb{R}^{d}$, $\|\cdot\|_2$ is the Euclidean norm, and $\|\cdot\|_{L_{loc}^{2}}$ is the local $L^2$-norm distance defined in \cite{chen1998stability}.
\item Furthermore, we also have an $L^p$-operator norm bound for $P_{t}-\widetilde{P}_{t}$ in terms of the local $L^{2}$-norm distance between $a_{i j}$ and $\tilde{a}_{i j}$:  There is a bounded, piecewise continuous function $F_{2}(t, z)$ on $\mathbb{R}^{+}  \times\mathbb{R}^{+} $ with $\lim _{z \rightarrow 0} F_{2}(t, z)=0$ for each $t>0$ that depends only on $d$ and  $\alpha$, such that for any $p \in[1, \infty]$, we have
\begin{align*}
   \left\|P_{t}-\tilde{P}_{t}\right\|_{p} \le F_{2}\left(t,\|A-\tilde{A}\|_{L_{l o c}^{2}}\right) .
\end{align*}
\end{enumerate}

\end{Theorem}
The above theorem shows the pointwise convergence of the heat kernel under perturbations of the matrix $A$. It proves the stability of the solution under small perturbations of the Laplace operator. These results can also be extended to general Riemannian manifolds \cite{sun2009concise} since every Riemannian manifold can be isometrically embedded into some Euclidean space \cite{nash1956imbedding} and the kernel is isometrically invariant \cite[theorem 9.12]{grigoryan2009heat}. We therefore know that if the difference between $\tilde{A}$ and $A$ (or more generally $\tilde{g}$ and $g$) are small, the final solutions to \cref{eq:heat_cauchy} have small difference. An intuitive explanation of the stability of the solution comes from the fact that the heat kernel $k_t( x, y)$ is the transition density function of a Brownian motion on the manifold \cite{hsu2002stochastic} since Brownian motion can be viewed as a diffusion generated by half of the Laplace operator. This means that $k_{t}(x, y)$ is a weighted average over all possible paths between $x$ and $y$ at time $t$, which does not change dramatically under small perturbations of $g$.  More specifically, for a Brownian motion on a domain $M$, if we perturb the $g$ over a subset $P \subset M$ then only the paths passing through $P$ will be affected.

If we go one step further from the uniformly elliptic operators $\operatorname{div}(A \nabla)$ defined in \cref{thm:bound1} by extending $A(x)\varphi$ to $A(x,t,\varphi,\nabla\varphi)$, we get the following  general \emph{quasilinear parabolic equation}:
\begin{align}
    \left\{\begin{array}{l}
\frac{\partial \varphi}{\partial t}=\operatorname{div} (A(x,t,\varphi, \nabla \varphi)),\ t>0 \\
\left.\varphi\right|_{t=0}=f,
\end{array}\right. \label{eq:heat_cauchy2}
\end{align}
where the principle part $A(x,t,\varphi, \nabla \varphi)$ is equipped with additional structure conditions as shown in \cite{dibenedetto2011harnack}. When $A(x,t,\varphi, \nabla \varphi)= A(x,t)$ is a time-dependent diffusion process, analogous bounds to \cref{thm:bound1} are shown in \cite{zheng1999stability}. We leave the analysis of general $A(x,t,\varphi, \nabla \varphi)$ for future work.

\subsection{Gradient and Divergence Operators on Graphs}\label{sec:grad_graph}
We consider a graph as a discretization of a Riemannian manifold. More specifically, consider an undirected graph $\calG = (\calV, \calE)$ consisting of a finite set $\calV$ of vertices, together with a subset $\calE \subset \calV \times \calV$ of edges. 
A graph is weighted when it is associated with a function $w : \calE \mapsto \Real^+$ which is symmetric, i.e., $w([u, v]) = w([v, u])$, for all $[u, v] \in \calE$.

Let $\calH(\calV)$ denote the Hilbert space on $\calV$ of real-valued functions with the inner product defined as $\langle a, b\rangle_{\calH(\calV)}=\sum_{v\in \calV}a(v)b(v)$, for all $a, b \in \calH(\calV)$. Similarly, we define a Hilbert space $\calH(\calE)$ with inner product $\langle c, d\rangle_{\calH(\calE)}=\sum_{[u,v]\in \calE}c([u,v])d([u,v])$, for all $c, d \in \calH(\calE)$.
We next define the gradient and divergence operators on graphs analogue to the general continuous manifold defined in \cref{def:div_gradient_manifold} as follows \cite{zhou2005regularization}.
\begin{Definition} Given an undirected graph $\calG = (\calV, \calE)$, we define the following:
\begin{enumerate}[i.]
    \item The graph gradient is an operator $\nabla: \calH(\calV)\mapsto\calH(\calE)$ defined by
\begin{align}\label{eqn:grad}
    \nabla \varphi([u,v])=\sqrt{\frac{w([u,v])}{h(v)}}\varphi(v) - \sqrt{\frac{w([u,v])}{h(u)}}\varphi(u), \forall [u,v]\in \calE,
\end{align}
where $h(v)=\sum_{[u,v]\in\calE}w([u,v])$  is the degree of node $v$.
\item The graph divergence is an operator ${\rm div}: \calH(\calE)\mapsto\calH(\calV)$ defined by 
\begin{align}\label{eqn:div}
    ({\rm div}\ \psi)(v)=\sum_{[u,v]\in\calE}\sqrt{\frac{w([u,v])}{h(v)}}\left(\psi([v,u]) - \psi([u,v])\right).
\end{align}
\item The graph Laplacian is an operator $\Delta: \calH(\calV)\mapsto\calH(\calV)$ defined by
\begin{align}\label{eqn:lap}
    \Delta\varphi=-\frac{1}{2}{\rm div}(\nabla\varphi). 
\end{align}
\end{enumerate}
\end{Definition}
The graph Laplacian defined above is identical to the normalized Laplacian matrix, i.e.,
\begin{align}
\Delta = \bD^{-1/2}(\bD - \bW)\bD^{-1/2},\label{eq:graphLaplacian}
\end{align}
 where $\bD$ is a diagonal matrix with $\bD(v, v) = h(v)$, and $\bW$ is the adjacency matrix satisfying $\bW(u, v) = w([u, v])$ if $[u, v]\in\calE$ and $\bW(u, v) = 0$ otherwise. Note that in \cref{eqn:lap} we have included the negative sign. The analogue of the graph Laplacian $\Delta$ in \cref{sec:heat_manifold} is the Dirichlet Laplace operator $\calL$. Note the use of the same notation $\Delta$ for graphs. The manifold Laplace operator $\Delta$ and graph Laplacian $\Delta$ will be apparent from the context.

\section{Neural Diffusion and Stability on Graphs}\label{sect:neural_gode}
We now consider neural diffusion on graphs by making use of concepts from \cref{sect:preliminaries}. Various parabolic-type equations on graphs are studied in this section. The solution stability against adversarial attacks is linked to the stability of the solution for the heat diffusion equation under perturbation of the manifold metric, introduced in \cref{sec:heat_manifold}. The general framework we consider is based on parabolic-type equations on graphs:
\begin{align}
\frac{\partial \varphi(u,t)}{\partial t}=\operatorname{div} (A(u,t,\varphi, \nabla \varphi)),\ t>0 \label{eq:heat_cauchy_graph}
\end{align}
with $\varphi(u,0)$ being the initial node attribute at node $u$ and where $A(u,t,\varphi, \nabla \varphi)$ can take different forms. We next provide some examples.

\subsection{Continuous Diffusion on Graphs}\label{ssect:diff_examples}
 \begin{Definition}[Heat diffusion]
The heat diffusion on graphs is defined by
\begin{align}
    \frac{\partial \varphi(u,t)}{\partial t}=\frac{1}{2}{\rm div}(\nabla \varphi)(u,t). \label{eq:heatgraph}
\end{align}
\end{Definition}
\begin{Definition}[GRAND/BLEND \cite{chamberlain2021grand,chamberlain2021blend}] \label{def:GRAND}
According to \cite[eq (1)]{chamberlain2021grand} and \cite[eq (7) and (9)]{chamberlain2021blend}, the GRAND/BLEND flow is defined by
\begin{align}\label{eqn:approx_curv}
    \frac{\partial \varphi(u,t)}{\partial t}=\frac{1}{2}{\rm div}(\nabla_{t} \varphi)(u,t),
\end{align}
where $\nabla_{t}$ is an adaptive Laplace operator (with possible graph rewiring) depending on the evolved node feature $\varphi(\cdot,t)$.
\end{Definition}
The gradient operator defined in GRAND/BLEND assumes constant edge weight. In \cref{def:GRAND}, the weight function $w([u,v])$ defined in \cref{eqn:grad} is incorporated into the gradient definition, so it can be absorbed in the time-dependent term $\nabla_{t} \varphi$. Note that BLEND degenerates to GRAND \cite{chamberlain2021grand} when there is no positional encoding. In this paper, for fair comparison, we do not use positional encoding for all GNNs.

Analogous to \cite{sochenTIP1998}, we can define mean curvature flow and Beltrami flow as follows.
\begin{Definition}
The mean curvature diffusion on graphs is defined by
\begin{align}\label{eqn:curv}
    \frac{\partial \varphi(u,t)}{\partial t}=\frac{1}{2}{\rm div}\left(\frac{\nabla \varphi}{\|\nabla \varphi\|}\right)(u,t),
\end{align}
where $-\dfrac{1}{2}{\rm div}\left(\frac{\nabla \varphi}{\|\nabla \varphi\|}\right)$ is a discrete analogue of the mean curvature operator, $\|\nabla \varphi\|=\langle \nabla \varphi, \nabla \varphi\rangle_{\calH(E)}^{1/2}$ and $\|\nabla \varphi(u,t)\|=\left(\sum_{[v,u]\in\calE} (\nabla \varphi([u,v],t))^2\right)^{1/2}$.
\end{Definition}

\begin{Definition}
The Beltrami diffusion on graphs is defined by
\begin{align}\label{eqn:belt}
    \frac{\partial \varphi(u,t)}{\partial t}=\frac{1}{2}\frac{1}{\|\nabla \varphi\|}{\rm div}\left(\frac{\nabla \varphi}{\|\nabla \varphi\|}\right)(u,t).
\end{align}
\end{Definition}

Intuitively, the term $\|\nabla \varphi(u,t)\|$, which appears in \cref{eqn:curv} and \cref{eqn:belt} but not in \cref{eq:heatgraph}, measures the smoothness of the signals in the neighborhood around vertex $u$ at time $t$.  The diffusion using \cref{eqn:curv} or \cref{eqn:belt} at vertex $u$ is small when $\|\nabla \varphi(u,t)\|$ is large, i.e., signals are less smooth around vertex $u$. Hence, mean curvature flow and Beltrami flow are able to preserve the non-smooth graph signals. This phenomenon is visualized in \cref{fig:tsne_flows}, where mean curvature flow and Beltrami flow are capable of preserving inter-class edges whose weights are much larger than those in heat flow.

\begin{figure}[!htb]
\centering
    \includegraphics[width=0.32\textwidth]{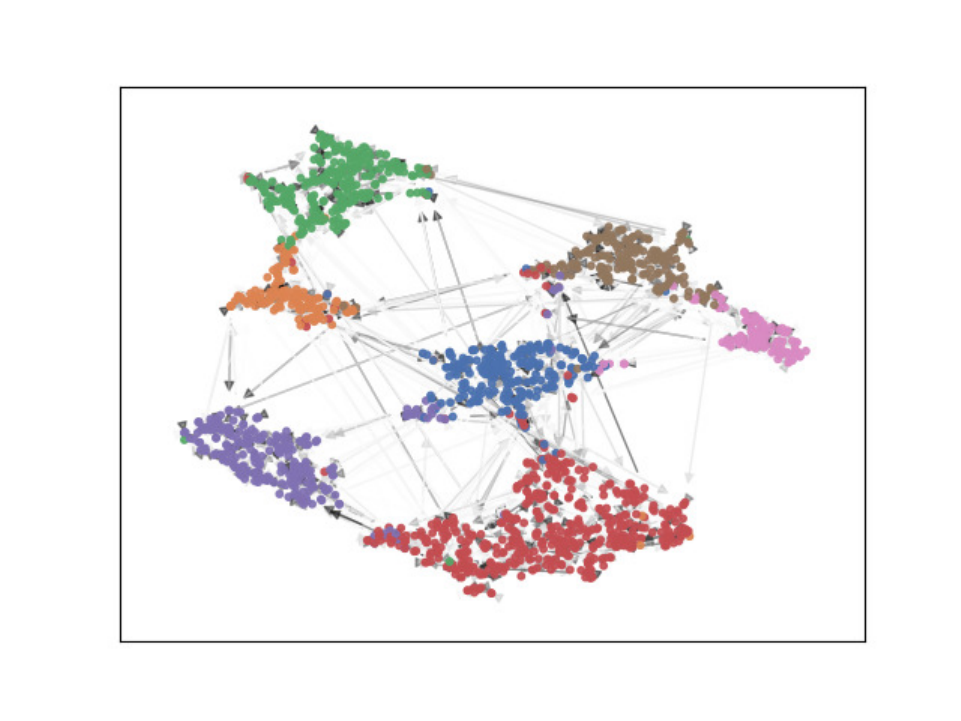}
    \includegraphics[width=0.32\textwidth]{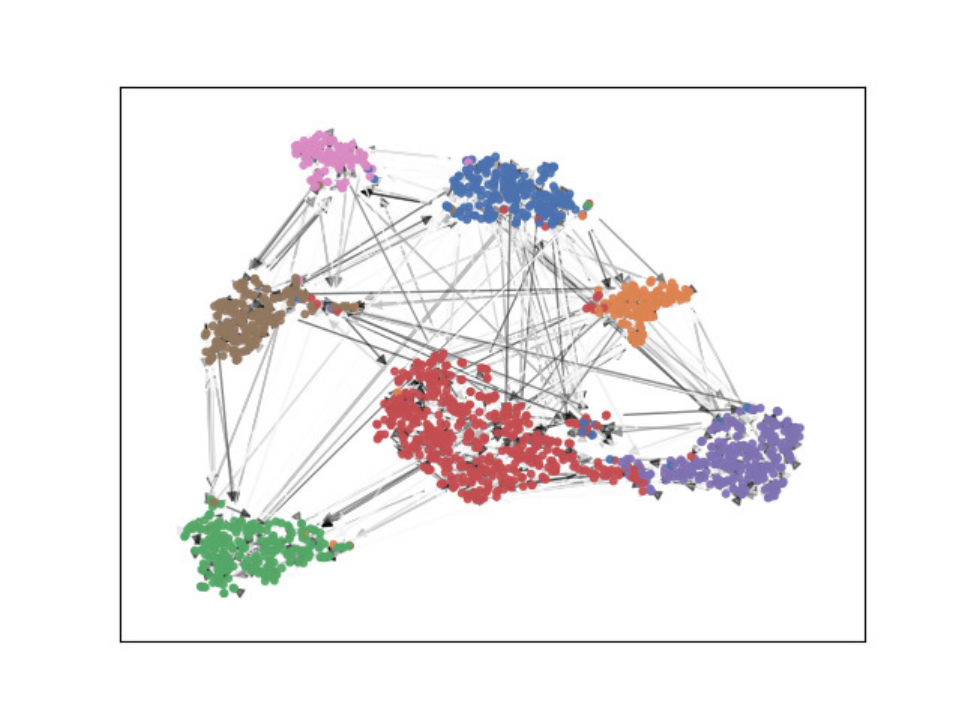}
    \includegraphics[width=0.32\textwidth]{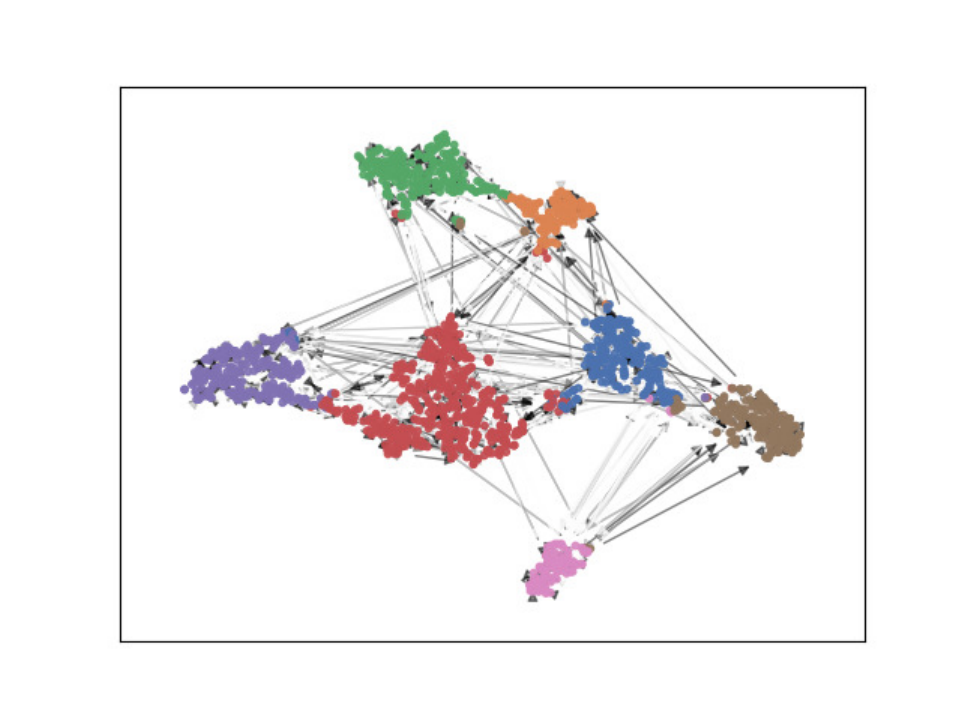}
    \vspace{-0.3cm}
  \caption{Effect of using different diffusion equations visualized using t-SNE with attention weights. Left: heat flow, middle: mean curvature flow, right: Beltrami flow. The darker the edge, the larger the attention weight.}
  \label{fig:tsne_flows}
\end{figure}
\subsection{Stability against Graph Topology Perturbation}
 Attackers can modify the original graph by adding or removing edges \cite{Chen2018FastGA,WaniekNHB2018,Du2017TopologyAG} and perturbing node attributes. Adding or removing edges leads to a different graph Laplacian $\Delta$ defined in \cref{eqn:lap}. Consequently, the solution to \cref{eq:heatgraph} is altered. Using the language of \cref{sec:heat_manifold}, the edge perturbations correspond to the perturbation of the metric $g$ (or equivalently $A$ in \cref{thm:bound1}). According to \cref{thm:bound1}, the stability of the heat semigroup and the solutions under small perturbations of the Laplace operator is guaranteed. Consequently, the solution is not affected significantly by the edge perturbations. 

The following result shows the stability of the solution of \cref{eq:heatgraph}, i.e., $ \frac{\partial \varphi(u,t)}{\partial t}= - \Delta \varphi(u,t)$, under graph topology or Laplacian perturbation. If the change in the graph topology with respect to any matrix norm is small, the semigroup perturbation can be bounded similarly as the result in \cref{thm:bound1}. This result considers only the case where the Laplace operator is time-invariant. The time-variant analogy of \cite{zheng1999stability} as discussed after \cref{eq:heat_cauchy2} is provided in the supplementary material. 

\begin{Proposition}\label{lem:linear_graph}
Consider $\Delta = \bD^{-1/2}(\bD - \bW)\bD^{-1/2}$ in \cref{eq:graphLaplacian} and suppose $\tilde{\Delta}=\tilde{\bD}^{-1/2}(\tilde{\bD} - \tilde{\bW})\tilde{\bD}^{-1/2}$ with $\tilde{\bD}$ being the diagonal degree matrix defined analogously to \cref{eq:graphLaplacian} for $\tilde{\bW}=\bW+ \bE$. Suppose $\varepsilon\coloneqq\vertiii{\bE}=o(1)$ for a matrix norm $\vertiii{}$, and $\bD$ and $\tilde{\bD}$ are non-singular. Then, $\|\varphi(u,t)-\tilde{\varphi}(u,t)\|= O(\varepsilon)$.
\end{Proposition}

\subsection{Neural Flows on Graphs}
Substituting \cref{eqn:grad} and \cref{eqn:div} into \cref{eqn:belt}, and ignoring the degree variable $h(u)$ for simplicity, we propose the neural Beltrami flow as
\begin{align}\label{eqn:neural_belt}
     \frac{\partial \varphi(u,t)}{\partial t}=\frac{1}{2}
     \frac{1}{\|\nabla \varphi(u)\|}
     \sum_{[v,u]\in\calE} w([u,v])
     \left(\frac{1}{\|\nabla \varphi(u)\|}+\frac{1}{\|\nabla \varphi(v)\|}\right)
     \left(\varphi(u)-\varphi(v)\right),
\end{align}
where $\|\nabla \varphi(u)\|=\sqrt{\sum_{[v,u]\in\calE} (\nabla \varphi([u,v]))^2}$. 
Here, we further assume $\varphi(u)$, for all nodes $u$ are time-independent for the sake of simplicity. 
Suppose $|\calV|=n$. The function $\varphi:\calV\mapsto\Real^d$ maps each node to a feature vector. Stacking all the feature vectors together, we obtain $\bZ\in\Real^{n\times d}$. Let the weight function $w([u,v])$ be the scaled dot product attention function \cite{velickovic2018graph} given by
\begin{align}\label{eqn:w}
    w([u,v]) = {\rm softmax}\left(\frac{(\bW_K\bz_u)\T(\bW_Q\bz_v)}{\sqrt{d_K}}\right),
\end{align}
where $\bW_K$ and $\bW_Q$ are the key and query learnable matrices, respectively, and $d_K$ denotes the number of rows $\bW_K$ has. If multi-head attention is applied and denote $\bA_h$ as the attention matrix associated with head $h$, then $\bA(\bZ)=\frac{1}{h}\sum_h \bA_h(\bZ)$ with $\bA_h(u,v)= w_h([u,v])$ where $w_h$ is the weight function for head $h$.  
The diffusion equation \cref{eqn:neural_belt} can be compactly written in matrix form as
\begin{align}\label{eqn:mtx_beltrami}
    \frac{\partial \bZ(t)}{\partial t}=\left(\bA(\bZ(t))\odot \bB(\bZ(t)) - \bPsi(\bZ(t))\right)\bZ(t),
\end{align}
where $\odot$ denotes element-wise multiplication, $\bB(u,v) = {\rm softmax}\left(\frac{1}{\|\nabla \varphi(u)\|^2}+\frac{1}{\|\nabla \varphi(u)\|\|\nabla \varphi(v)\|}\right)$ and $\bPsi(\bZ(t))$ is a diagonal matrix with $\Psi(u,u) = \sum_v \left(\bA\odot \bB\right)(u, v)$. 

Similarly, the diffusion equations using mean curvature flow \cref{eqn:curv} and heat flow \cref{eq:heatgraph} can be written~as:
\begin{align}\label{eqn:mtx_mean}
    \frac{\partial \bZ(t)}{\partial t}=\left(\bA(\bZ(t))\odot \bB(\bZ(t)) - \bPsi(\bZ(t))\right)\bZ(t),
\end{align}
where $\bB(u,v) = {\rm softmax}\left(\frac{1}{\|\nabla \varphi(u)\|}+\frac{1}{\|\nabla \varphi(v)\|}\right)$ and
\begin{align}\label{eqn:mtx_heat}
    \frac{\partial \bZ(t)}{\partial t}=(\bA(\bZ(t)) - \bI)\bZ(t),
\end{align}
respectively. 
Although the BLEND model in \cite{chamberlain2021blend} is inspired from Beltrami flow, its final formulation (equation (8) in \cite{chamberlain2021blend} where $\bZ(t)$ contains positional and node feature embeddings) is indeed \cref{eqn:mtx_heat}, i.e., heat flow using attention weight function.

\begin{Proposition}\label{lem:lyap_stable}
Diffusion equations \cref{eqn:mtx_beltrami}, \cref{eqn:mtx_mean} and \cref{eqn:mtx_heat} are all bounded-input bounded-output (BIBO) stable \cite{{chen1999linear}}.
\end{Proposition}

\section{Model Architecture with Lipschitz Constraint}
\label{sect:model}

In this section, we propose a general graph neural PDE network based on the discussion in \cref{sect:neural_gode}.
A layer of graph neural PDE is illustrated in \cref{fig:gode_layer}. Since we may stack up multiple such layers, the diffusion is performed in a hierarchical manner, where the high-dimensional features are diffused over the graph at the front layers and the low-dimensional features are diffused at the back layers.
\begin{figure}[!htb]
\centering
    \includegraphics[width=0.6\textwidth]{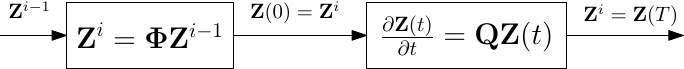}
  \caption{Graph neural PDE at the $i$-th layer, where each node's feature vector are linearly transformed before being diffused over the graph. We have $\bQ = \bA(\bZ(t))\odot \bB(\bZ(t)) - \bPsi(\bZ(t))$ for mean curvature flow and Beltrami flow, while $\bQ = \bA(\bZ(t)) - \bI$ for heat flow.}
  \label{fig:gode_layer}
\end{figure}
One important ingredient in our model is that we perform spectral normalization on $\bW_K$ and $\bW_Q$ in \cref{eqn:w}. This is motivated by the fact that attention models have poor performance when the depth increases. In other words, attention weights tend to be uniformly distributed for excessive message exchanges. This phenomenon becomes even more obvious when the nodes' features are diffused according to \cref{eqn:mtx_beltrami}, \cref{eqn:mtx_mean} or \cref{eqn:mtx_heat} as solving such a PDE normally requires many discrete steps. To overcome this over-smoothing problem, we utilize the strategy proposed in \cite{lipschitznorm_icml21} to enforce Lipschitz continuity by normalizing the attention scores. Besides, enforcing Lipschitz continuity can also help improve the robustness of the model because the Lipschitz constant controls the perturbation of the output given a bounded input perturbation. \cref{fig:tsne_q2} shows that attention weights become overly smooth if no spectral normalization is applied. 
\begin{figure}[!htp]
\vspace{-0.3cm}
\centering
    \includegraphics[width=0.38\textwidth]{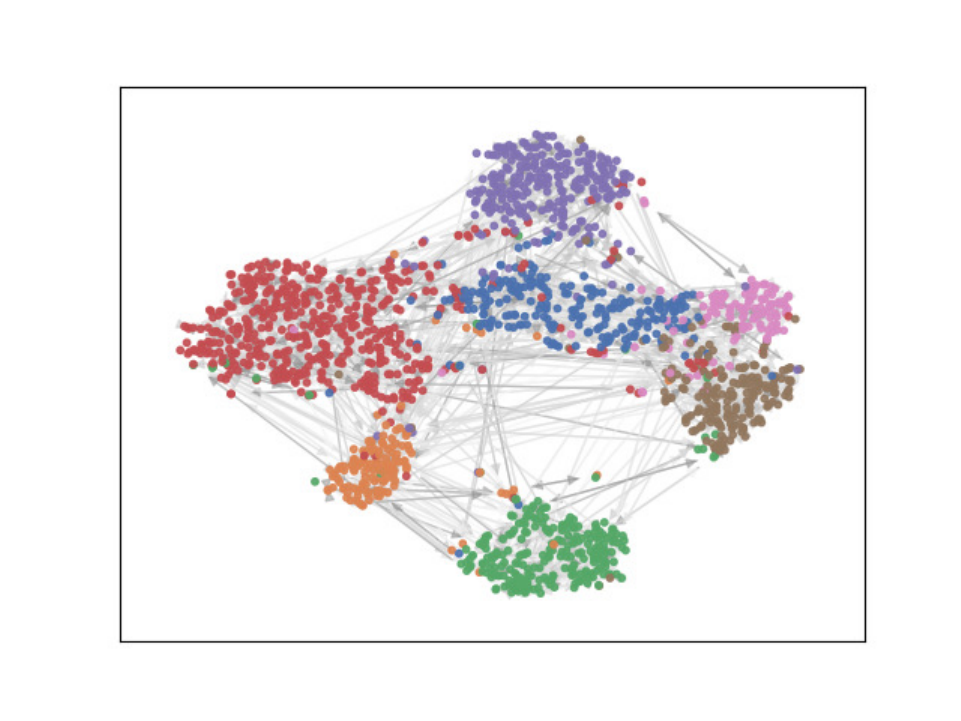}
    \includegraphics[width=0.38\textwidth]{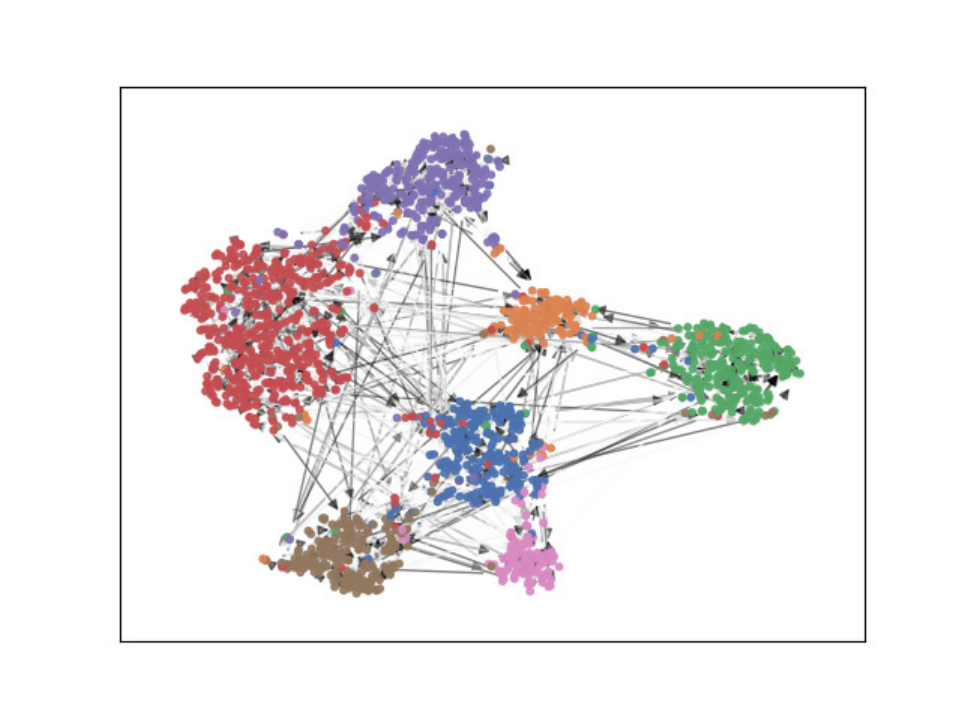}
    \vspace{-0.2cm}
  \caption{Impact of applying spectral normalization visualized using t-SNE with attention weights using Cora dataset \cite{McCallum2004AutomatingTC}. Left: graph neural PDE with no spectral normalization, right: graph neural PDE with spectral normalization. The darker the edge, the larger the attention weight.}
  \label{fig:tsne_q2}
\end{figure}
\section{Experiments}
\label{sect:exper}

In this section, we compare graph neural PDEs under different flows to popular GNN architectures: GAT \cite{velickovic2018graph}, GraphSAGE \cite{hamilton2017inductive}, GIN \cite{xu2018how}, APPNP \cite{KlicperaBG19}, and the state-of-the-art GNN defenders: RobustGCN \cite{zhuKDD2019}, GNNGuard \cite{zhangNeurips2020}, GCNSVD \cite{EntezariWSDM2020}, on standard node classification benchmarks. 
In our experiments\footnote{Our experiments are run on a GeForce RTX 3090 GPU.}, we use the following datasets: Cora (citation networks) \cite{McCallum2004AutomatingTC}, Citeseer (citation networks) \cite{SenNamata2008} 
 and PubMed (biomedical literature) \cite{namata:mlg12-wkshp}.
We use a refined version of these datasets provided by \cite{zheng2021grb}. We refer the readers to the supplementary material for more details.   Our experiment codes are provided in \url{https://github.com/zknus/Robustness-of-Graph-Neural-Diffusion}.

\begin{table}[!htp]
\centering
\caption{Node classification accuracy (\%) on adversarial examples using different GNNs. The implicit Adam PDE solver with step size 2 is used for Beltrami. We denote those experiments that are computationally too heavy to run by ``-''.  The best and the second-best result for each criterion are highlighted in \red{red} and \blue{blue} respectively.} \label{tab:adv} 
\tiny
\makebox[\textwidth][c]{
\begin{tabular}{cccccccccc} 
\toprule
Dataset & Attack & Beltrami  in \cref{eqn:mtx_beltrami} & RobustGCN & GNNGuard & GCNSVD  & GAT & GraphSAGE & GIN & APPNP \\
\midrule
\multirow{3}{*}{Cora} 
& \emph{clean} & 75.93 $\pm$ 1.46  & \blue{81.34 $\pm$ 0.66} & 79.44 $\pm$ 1.18 & 69.28$\pm$1.37  & 79.74 $\pm$ 1.59 & 76.75 $\pm$ 1.52 & 76.79 $\pm$ 1.35 & \red{83.06 $\pm$ 1.06} \\
& SPEIT &  \blue{61.87 $\pm$ 0.49} & 36.16 $\pm$ 0.41 & \red{78.50 $\pm$ 2.27} & 37.50 $\pm$ 0.74  & 38.10 $\pm$ 2.48 & 35.82 $\pm$ 0.01 & 35.82 $\pm$ 0.01 & 36.79 $\pm$ 0.61 \\
& TDGIA &  \blue{62.84 $\pm$ 1.17} & 53.28 $\pm$ 8.61 & \red{78.92 $\pm$ 1.80} & 40.77 $\pm$ 3.34  & 35.64 $\pm$ 12.91 & 39.78 $\pm$ 6.46 & 39.63 $\pm$ 2.38 & 60.52 $\pm$ 4.43 \\
\midrule

\multirow{3}{*}{Citeseer} 
& \emph{clean} & 70.14 $\pm$ 1.80 & \blue{70.72 $\pm$ 1.15} & 69.69 $\pm$ 1.83 & 66.93 $\pm$ 1.07  & 69.81 $\pm$ 1.43 & 69.78 $\pm$ 1.31 & 68.81 $\pm$ 1.58 & \red{70.75 $\pm$ 0.86} \\
& SPEIT & \blue{66.46 $\pm$ 1.33} & 28.56 $\pm$ 7.87 & \red{69.72 $\pm$ 1.84} & 21.16 $\pm$ 1.32  & 26.00 $\pm$ 11.14 & 19.75 $\pm$ 1.82 & 23.54 $\pm$ 5.30 & 22.19 $\pm$ 0.86 \\
& TDGIA & \blue{65.77 $\pm$ 1.28} & 38.81 $\pm$ 10.84 & \red{69.50 $\pm$ 1.86} & 20.77 $\pm$ 2.52 & 19.63 $\pm$ 6.53  & 28.77 $\pm$ 7.73 & 28.65 $\pm$ 5.08 & 54.48 $\pm$ 8.56  \\
\midrule

\multirow{3}{*}{PubMed} 
& \emph{clean} & \blue{86.94 $\pm$ 0.25} & 75.55 $\pm$ 0.32 & 84.80 $\pm$ 0.51 & - & 84.91 $\pm$ 0.76 & \red{89.22 $\pm$ 0.25} & 76.71 $\pm$ 0.14 & 77.50 $\pm$ 0.54  \\
& SPEIT & \red{86.66 $\pm$ 0.68} & 75.54 $\pm$ 0.54 & \blue{84.36 $\pm$ 0.58} & - & 40.94 $\pm$ 2.47 & 39.22 $\pm$ 0.00 & 76.71 $\pm$ 0.14 & 77.55 $\pm$ 0.54  \\
& TDGIA & \red{85.56 $\pm$ 0.91} & 75.53 $\pm$ 0.36 & \blue{84.00 $\pm$ 1.12} & - & 39.78 $\pm$ 0.29 & 60.40 $\pm$ 11.23 & 77.58 $\pm$ 0.71 & 77.45 $\pm$ 0.68  \\

\bottomrule
\end{tabular}}
\vspace{-0.35cm}
\end{table}

\subsection{Attack Setup}
We apply the setup introduced in the graph robustness benchmark (GRB) \cite{zheng2021grb}. Based on the assumption that nodes with lower degrees are easier to attack,  GRB constructs three test subsets of nodes with different degree distributions. According to the average degrees, GRB defines these subsets of nodes as Easy, Medium, Hard, or Full. In the Easy subset, attacks are easy to succeed and hence the worst performance is expected. For the rest of the nodes, they are divided into a train set (60\%) and val set (10\%), for training and validation respectively. In our experiments, we choose the Easy subset, i.e., the most challenging mode.

Following GRB, in this paper, we mainly consider the following real-world adversarial attack settings: 
{\em Black-box}: The attacker does not know the defender's method and vice versa. Attackers firstly attack a pre-trained GCN \cite{kipf2017semi} and then transfer the perturbed graphs to the target model.
{\em Evasion}: Attacks will only happen during the inference phase.
{\em Inductive}: GNNs are used to classify unseen data (e.g., new users), i.e., validation or test data are unseen during training.
{\em Injection}: Attackers can only inject new nodes but not modify the target nodes directly. This reflects application scenarios like in online social networks where it is usually hard to hack into users' accounts and modify their profiles. However, it is easier to create fake accounts and connect them to existing users. For other attack settings such as white-box attacks and modification attacks, we have included relevant experiments and discussions in the supplementary material.

\begin{table}[!tbh]
\centering
\caption{Node classification accuracy (\%) on adversarial examples using BeltramiGuard.} \label{tab:BelGuard}
\vspace{-6pt}
\small
\begin{tabular}{ccccc} 
\toprule
Model & Attack & Cora & Citeseer & PubMed \\
\midrule
\multirow{3}{*}{BeltramiGuard} 
& \emph{clean} & 73.01 $\pm$ 2.01  & 69.90 $\pm$ 0.44 & 87.77 $\pm$ 0.14   \\
& SPEIT        & 73.01 $\pm$ 2.62  & 69.90 $\pm$ 0.44 & 87.68 $\pm$ 0.03   \\
& TDGIA        & 72.14 $\pm$ 1.20  & 69.90 $\pm$ 0.44 & 88.16 $\pm$ 0.61    \\
\bottomrule
\end{tabular}
\end{table}

We apply two state-of-the-art injection attack methods: SPEIT \cite{speit_attack} and Topological Defective Graph Injection Attack (TDGIA) \cite{zouKDD2021}. These two attacks are the two strongest attacks reported in the GRB. For both attacks, both the maximum allowable injected nodes and the maximum allowable injected edges are set to 50 for Cora and Citeseer, and 300 for PubMed. We provide more details about how these two attacks work under evasion, black-box and injection setting in the supplemental material.

\cref{tab:adv} shows that graph neural PDE induced from Beltrami flow is more robust than all other GNNs except for GNNGuard, which is specifically designed to remove malicious edges and is thus robust against topology perturbation.  This shows that the output features from a graph PDE are stable under topology perturbation, as suggested by \cref{lem:linear_graph}.

Furthermore, the heat diffusion process on a manifold tends to diffuse differently under different geometries. For example, it diffuses slower at points with positive curvature, and faster at points with negative curvature \cite{sun2009concise,bhaskar2022diffusion}. In \cite{ToppingICLR2022} and \cite{ZhuNIPS2020}, the authors show different datasets have different geometric properties like hyperbolicity distribution or Balanced Forman curvature. Our theoretical analysis only shows a loose uniform bound in terms of the adjacency matrix, while the performance is very much dataset-dependent. More advanced theoretical analysis for different datasets is highly non-trivial and needs further extensive investigations.

Our graph neural PDE can be combined with GNNGuard. We denote this model as BeltramiGuard. \cref{tab:BelGuard} shows that BeltramiGuard render attacks in vain and exceeds GNNGuard on the Citeseer and PubMed datasets. Note from \cref{tab:adv} that the vanilla Beltrami model already surpasses GNNGuard on the large PubMed. More details are provided in the supplementary material.

\subsection{Ablation Studies} \label{sec:abla}

\subsubsection*{Diffusion Schemes} 
We compare different diffusion equations: heat flow \cref{eqn:mtx_heat} where $w([u,v])=1$ for all $[u, v] \in \calE$, GRAND/BLEND 
\cite{chamberlain2021grand,chamberlain2021blend} which is another heat flow where $w([u,v])$ in \cref{eqn:mtx_heat} is the attention function defined in \cref{eqn:w}, mean curvature flow \cref{eqn:mtx_mean}, and Beltrami flow \cref{eqn:mtx_beltrami}. For GRAND/BLEND, we stack three neural PDE layers using the architecture proposed in \cref{fig:gode_layer} since the original GRAND/BLEND in \cite{chamberlain2021grand,chamberlain2021blend}, which has only one PDE layer, does not perform well.

\begin{table}[!tbh]
\centering
\caption{Node classification accuracy (\%) on adversarial examples using graph neural PDEs induced from different flows, where implicit Adam PDE solver with step size 2 is used.  } \label{tab:flows}
\small
\begin{tabular}{cccccc} 
\toprule
Dataset & Attack & Heat & GRAND/BLEND & Mean Curvature & Beltrami \\
\midrule
\multirow{3}{*}{Cora} 
& \emph{clean} & \red{78.86 $\pm$ 1.78}  &  74.89 $\pm$ 1.29  & \blue{76.01 $\pm$ 2.20}  &  75.93 $\pm$ 1.46 \\
& SPEIT        & 38.66 $\pm$ 2.40  &  55.67 $\pm$ 3.60  & \blue{60.67 $\pm$ 1.31}  &  \red{61.87 $\pm$ 0.49}  \\
& TDGIA        & 60.26 $\pm$ 5.19  &  59.55 $\pm$ 6.41  & \blue{62.01 $\pm$ 2.37}  &  \red{62.84 $\pm$ 1.17}   \\
\midrule

\multirow{3}{*}{Citeseer} 
& \emph{clean} & 69.47 $\pm$ 1.22  &  69.45 $\pm$ 0.91  & \red{70.50 $\pm$ 1.63}  &  \blue{70.14 $\pm$ 1.80} \\
& SPEIT        & 22.95 $\pm$ 5.07  &  39.94 $\pm$ 6.94  & \blue{65.39 $\pm$ 1.62}  &  \red{66.46 $\pm$ 1.33}  \\
& TDGIA        & 52.42 $\pm$ 11.2  &  54.92 $\pm$ 7.64  & \red{66.83 $\pm$ 1.59}  &  \blue{65.77 $\pm$ 1.28}   \\
\midrule

\multirow{3}{*}{PubMed} 
& \emph{clean} &  86.31 $\pm$ 0.46 &  \red{88.89 $\pm$ 0.38}  & \blue{88.45 $\pm$ 0.32}  &  86.94 $\pm$ 0.25 \\
& SPEIT        &  40.77 $\pm$ 1.78 &  39.39 $\pm$ 0.26  & \red{87.13 $\pm$ 0.33}  &  \blue{86.66 $\pm$ 0.68}  \\
& TDGIA        &  42.63 $\pm$ 5.28 &  63.91 $\pm$ 11.61  & \red{85.79 $\pm$ 0.82}  &  \blue{85.56 $\pm$ 0.91}   \\

\bottomrule
\end{tabular}
\end{table}

From \cref{tab:flows}, we observe that even the vanilla time-invariant heat flow preserves some robustness as compared to non-PDE GNNs in \cref{tab:adv}. This further validates our theoretical analysis in \cref{lem:linear_graph}, which suggests that if the topology perturbation is bounded, the learned representations are close to those under the ``clean'' scenario. We can observe that 1) the flows which are capable of preserving non-smooth features are generally more robust than heat flow, 2) the proposed mean curvature flow and Beltrami flow are more robust than GRAND/BLEND,  and 3) these two flows generally suffer less performance variance than the other methods.

\subsubsection*{More Experiments}
Due to space constraint, we refer the reader to the supplementary material for more ablation studies including the impact of the Lipschitz constraint, PDE solvers, the number of layers, and time complexity. The model performance under other attacks is also presented in the  supplementary~material.
\section{Conclusion}\label{sect:conc}
In this paper, we have introduced a general graph neural PDE framework from which several graph neural PDEs are proposed. We analyzed the robustness of the graph neural PDEs and showed that graph neural PDEs are inherently robust against topology perturbations and are also BIBO stable. We provided theoretical evidence showing that the robustness of graph neural PDEs stems from the stability of the heat kernel and semigroup during the diffusion process. Moreover, we conducted extensive experiments that empirically verify the robustness of the proposed graph neural PDEs when compared with an existing graph neural PDE induced from approximated heat flow, popular GNNs and the state-of-the-art GNN defenders.

\appendix

\textbf{\LARGE Appendix}

In \cref{sec:supp_dataset} of this supplementary material, we provide more details about the datasets used in our main paper. 
We brief the attack settings in \cref{sec:att_set}. More experiments and ablation studies that are not included in the main paper due to space limit are now presented in \cref{sec:more_abl_exp}. 
We include supplemental experiments with more datasets in \cref{sec:supp_more_exp}. Further ablation studies of our model are also presented in \cref{sec:supp_ablation}.
In addition, we present an extension of mean curvature and Beltrami flows in \cref{sec:supp_extension}. Additional implementation details of the models are provided in \cref{sec:supp_imp_det}. 
The proofs for all theoretical results in the main paper are given in \cref{sec:supp_proof}.

\section{Datasets}\label{sec:supp_dataset}
In our main paper, we conduct experiments using the first three datasets in \cref{tab:data}: Cora (citation networks) \cite{McCallum2004AutomatingTC}, Citeseer (citation networks) \cite{SenNamata2008} and PubMed (biomedical literature) \cite{namata:mlg12-wkshp}. 
In this supplementary material, we conduct further experiments using the other three datasets in \cref{tab:data}: Flickr (social networks) \cite{McAuleyECCV2012}, Coauthor (academic networks) \cite{shchur2018pitfalls} and Amazon Computer (recommendation networks) \cite{McAuleySIGIR2015}.
We use a refined version of these datasets provided by \cite{zheng2021grb}, where the main statistics are summarized in \cref{tab:data}.  The features in these datasets are normalized by an $\arctan$ (bijective) transformation \cite{zheng2021grb}, which permits attackers to restore the original features and thus allows real-world adversarial attacks.

\begin{table}[!tbh]
\centering
\caption{Statistic of Datasets} \label{tab:data} 
\small
\begin{tabular}{ccccccc} 
\toprule
Dataset &  \# Nodes & \# Edges & \# Features & \# Classed & Feature Range (norm) \\
\midrule

 Cora 	   & 2,680  & 5,148  & 302  & 7  & $-0.94\sim0.94$  \\
\midrule

 Citeseer 	   & 3,191  & 4,172	  & 768  & 6  & $-0.96\sim0.89$  \\
\midrule

 PubMed 		   & 19,717  & 44,325	  &  500  & 3  &   $-0.14\sim0.99$ \\
\midrule

  Coauthor 	   & 18,333  & 81,894 	  & 6,805  & 15  &  $-0.04\sim1.00$ \\
 \midrule
 
 Flickr 	   & 89,250  & 449,878	  & 500  & 7  & $-0.47\sim1.00$  \\
 \midrule

  
  Amazon Computer 	   & 13,752  & 245,861	  & 767  & 10  & $-0.40\sim0.60$  \\
\bottomrule
\end{tabular}
\end{table}

\section{Attack Settings}\label{sec:att_set}
There are three common categories of adversarial attacks studied in the literature \cite{zheng2021grb}:
\begin{itemize}
    \item  Poison (attack occurs in training) or evasion (attack occurs in testing). 
    \item White-box (attackers knows target model/method) or black-box (attackers do not know target model/method and thus need to attack a surrogate model and then transfer to target model).
    \item Injection (attackers inject nodes/edges to the original graph and generate attributes for the injected nodes) or modification (attackers modify the original graph including its topology and node features directly).
\end{itemize}

Our experiments in the main paper mainly focus on the evasion, black-box and injection attack setting, which we believe are the most realistic attack settings. The two selected attack methods SPEIT and TDGIA are tailored for these attack settings. To be more specific, we carry out the following:
\begin{itemize}
    \item  Evasion: SPEIT and TDGIA are performed on a trained model during testing time.
    \item  Black-box: SPEIT and TDGIA are used to attack a trained GCN, i.e., a surrogate model, to generate graph perturbations and then the target model is tested on this perturbed graph.
    \item Injection: when perturbing the graph based a trained GCN, SPEIT and TDGIA firstly inject new nodes into the original graph and then generate the injected nodes' features. 
\end{itemize}
In this supplementary material, we include more attack settings such as white-box attacks and modification attacks. Further experiments to test graph PDE robustness against  node attribute perturbation are also included.

\section{More Experiments and Ablation Studies}\label{sec:more_abl_exp}

We include more experiments and ablation studies which are not included in the main paper due to space limit. More specially, in this section, we conduct experiments to study the effect of  Lipschitz constraint and the number of layers,  inference time complexity, more injection attacks including evasion white-box injection attack and the attacks with various attack strengths, and the modification attackers which modify the original graph including its topology and node features directly.

Our codes are developed based on the following two repositories:
\begin{itemize}
    \item \url{https://github.com/twitter-research/graph-neural-pde} and
    \item \url{https://github.com/THUDM/grb},
\end{itemize}
where the new diffusion schemes and their induced neural PDEs are developed based on the first repository and we follow the second repository to set up the robustness evaluation benchmark.

\subsection{PDE Solvers} 
The impact of PDE solvers on the performance can be observed in \cref{tab:solvers_rev}. The adaptive step-size solver Dopri5 performs better than the fixed-step solvers (Implicit/Explicit Adam) at the cost of higher computational complexity. For fixed-step solvers, increasing the step size $\tau$ reduces the variance. For sufficiently large step sizes, the implicit method converges faster than the explicit method.

\begin{table}[!tbh]
\centering
\caption{Node classification accuracy (\%) using graph neural PDEs induced from Beltrami flow, when different PDE solvers are applied. Experiments are conducted on Citeseer dataset.} \label{tab:solvers_rev}
\begin{tabular}{ccccc} 
\toprule
PDE solvers & Param. & Clean & SPEIT & TDGIA \\
\midrule
\multirow{2}{*}{Implicit Adam} 
& $\tau=1$  &  70.59  $\pm$  2.26   &  64.64   $\pm$    2.60  & 65.62  $\pm$  0.96 \\
& $\tau=2$  &  70.14  $\pm$  1.80    &   66.46  $\pm$  1.33  & 65.77  $\pm$  1.28 \\
& $\tau=10$ &  70.22  $\pm$   0.70  &  64.14   $\pm$   1.00   & 65.75  $\pm$  1.65 \\
\midrule

\multirow{2}{*}{Explicit Adam} 
& $\tau=1$  & 69.72  $\pm$  1.14  &  62.88  $\pm$  2.70  & 64.81  $\pm$  1.62  \\
& $\tau=2$  &  69.59  $\pm$  0.92    &   65.05  $\pm$  1.39  &  65.38  $\pm$  1.35  \\
& $\tau=10$ & 69.91  $\pm$  0.96  &  64.76  $\pm$  1.34  & 65.52  $\pm$  3.17  \\
\midrule

Dopri5  & -
& 70.91   $\pm$  0.98   &  66.96   $\pm$  1.46   & 67.01 $\pm$ 1.94 \\

\bottomrule
\end{tabular}
\end{table}

\subsection{Lipschitz Constraint and Number of Layers}

On one hand, keeping as few layers as possible can mitigate over-smoothing problem. On the other hand, too few layers result in underfitting problem, i.e., the test clean accuracy is low. To understand this better, we performed experiments using different number of layers. The results are summarised in \cref{tab:layers}. The number of layers to use is a hyperparameter we tuned during training.

\begin{table}[!t]
\small
\centering
\caption{Node classification accuracy (\%) on adversarial examples generated from SPEIT. We apply graph neural PDEs induced from Beltrami flow with or without the Lipschitz constraint. Experiments are conducted on Cora dataset.} \label{tab:lips} 
\begin{tabular}{cccccc} 
\toprule
Clean/Robust acc. & \# layers used & Beltrami  & Beltrami w/o Lips.\\
\midrule
\multirow{4}{*}{Clean} 
& 1 &  63.96 $\pm$ 1.95  &  63.51 $\pm$ 3.18   \\
& 2 &  72.46 $\pm$ 0.76   &  69.93 $\pm$ 1.48     \\
& 3 &  73.51 $\pm$ 1.49   &   72.69 $\pm$ 1.89    \\
& 4 &  75.93 $\pm$ 1.46  &   75.15 $\pm$ 1.41  \\
\midrule

\multirow{4}{*}{Robust} 
& 1 &  59.10 $\pm$ 3.20  &  58.06 $\pm$ 1.78  \\
& 2 &  63.58 $\pm$ 1.39   &  58.96 $\pm$ 2.82   \\
& 3 &  60.52 $\pm$ 1.56   &   57.46 $\pm$ 1.65  \\
& 4 &  61.87 $\pm$ 0.49  &   56.79 $\pm$ 1.52   \\
\bottomrule
\label{tab:layers}
\end{tabular}
\end{table}

We have tried using the same number of layers for GRAND/BLEND as in their paper. However, the test clean accuracy is low. This is mainly due to two reasons: 1) we are using inductive setting whereas the paper of GRAND/BLEND uses transductive training; 2) the datasets we are using have been calibrated for the robustness evaluation, which is different from the original datasets used by the paper of GRAND/BLEND. For example, grb-cora \cite{zheng2021grb} has node feature size of 302 while the original Cora dataset has node feature size of 1433. \cref{tab:layers} provides empirical evidence for GRAND/BLEND needing more layers in our setting.


\subsection{Time Complexity}

We have summarized the time complexity in \cref{tab:time}. The time is computed by averaging 500 diffusion operations. We can see that Beltrami and GRAND/BLEND have similar time complexity to the two defenders GNNGuard and GCNSVD small step sizes but incur more computation time than the other GNNs to complete a diffusion process. Heat, unlike Beltrami and GRAND/BLEND, does not use attention and has the lowest time complexity among all the neural PDEs. 

\begin{table}[!tbh]
\centering
\caption{Top: Average time spent on a Beltrami diffusion process, i.e., time to solve \cref{eqn:mtx_beltrami} and a counterpart in GRAND/BLEND, when different PDE solvers are applied and multiple step size options are tested for each solver. Bottom: Average time spent on an aggregation step using different GNNs. Experiments are conducted on the Citeseer dataset.} \label{tab:time}
\begin{tabular}{c cc cc cc} 
\toprule
PDE solvers & Param. & Beltrami   & GRAND/BLEND  & Heat\\
\midrule
\multirow{3}{*}{Implicit Adam} 
& $\tau=1$    &  9.8ms & 6.6ms & 3.0ms\\
& $\tau=2$    &  17.0ms &  11.2ms & 4.0ms\\
& $\tau=10$   &  48.1ms & 46.6ms & 9.2ms\\
\midrule

\multirow{3}{*}{Explicit Adam} 
& $\tau=1$    & 10.0ms  &  6.8ms & 3.0ms\\
& $\tau=2$    &  16.8ms &  11.2ms & 3.6ms\\
& $\tau=10$   & 32.6ms  &  21.0ms & 5.8ms\\
\midrule

Dopri5  & - & 66.0ms  & 20.0ms  & 13.0ms\\

\bottomrule
\end{tabular}

\vspace{.5cm}

\begin{tabular}{ccccccc} 
\toprule
RobustGCN & GNNGuard & GCNSVD  & GAT & GraphSAGE & GIN & APPNP \\
\midrule
0.6ms & 13.2ms  &9.0ms   & 1.8ms   &0.8ms   &  1.0ms & 1.6ms \\
\bottomrule

\end{tabular}
\end{table}

\subsection{White Box Attacks}

We now include the white box attacks for our model and the baselines. The results are shown in \cref{tab:white-box}. We observe that our model outperforms the baselines under white box attacks. The results further validate that neural PDEs are intrinsically robust to graph topology perturbations, as indicated by \cref{lem:linear_graph} and \cref{lem:linear_graph_timevar}.

\begin{table}[!t]
\centering

\caption{Node classification accuracy (\%) on adversarial examples generated from \emph{white-box} attacks. The best and the second-best result for each criterion are highlighted in \red{red} and \blue{blue}, respectively.}  \label{tab:white-box} 
\tiny

\makebox[\textwidth][c]{
\begin{tabular}{ccccccccccc} 
\toprule
Dataset & Attack & BeltramiGuard & Beltrami  & RobustGCN & GNNGuard & GCNSVD  & GAT & GraphSAGE & GIN & APPNP \\
\midrule
\multirow{3}{*}{Cora} 
& \emph{clean} &  73.01 $\pm$ 2.01     &   75.93 $\pm$ 1.46   &   \blue{81.34 $\pm$ 0.66}  &  79.44 $\pm$ 1.18   &   69.28 $\pm$ 1.37   &   79.74 $\pm$ 1.59   &  76.75 $\pm$ 1.52   &  76.79 $\pm$ 1.35     &  \red{83.06 $\pm$ 1.06}\\
& SPEIT &  \blue{71.94 $\pm$ 0.31}  &  48.95 $\pm$ 3.32  &  36.12 $\pm$ 0.31  & \red{80.22 $\pm$ 0.91}   &  33.06 $\pm$ 7.86   &  19.18 $\pm$ 9.47   & 17.16 $\pm$ 11.30   &  21.49 $\pm$ 13.08    & 19.77 $\pm$ 10.39 \\
& TDGIA &  \blue{67.39 $\pm$ 1.97}  &  52.61 $\pm$ 4.20 &   36.27 $\pm$ 0.61  &  \red{78.43 $\pm$ 1.10}   &  11.94 $\pm$ 0.0  & 7.99 $\pm$ 4.40   & 21.57 $\pm$ 4.82   &  38.06 $\pm$ 4.69    & 51.94 $\pm$ 5.32 \\
\midrule

\multirow{3}{*}{Citeseer} 
& \emph{clean} &   69.90 $\pm$ 0.44  &   \blue{70.41 $\pm$ 1.38}  &   70.72 $\pm$ 1.15  &  69.69 $\pm$ 1.83   &   66.93 $\pm$ 1.07   &   69.81 $\pm$ 1.43   &  69.78 $\pm$ 1.31   &   68.81 $\pm$ 1.58    &  \red{70.75 $\pm$ 0.86} \\
& SPEIT &  \blue{68.78 $\pm$ 0.82}  & 55.24 $\pm$ 6.90  &  20.19 $\pm$ 2.61 &  \red{69.22 $\pm$ 1.90}  &  19.31 $\pm$ 3.55  & 14.67 $\pm$ 5.05  & 19.81 $\pm$ 3.06   &  12.54 $\pm$ 6.33   & 20.75 $\pm$ 2.32    \\
& TDGIA &  \blue{67.96 $\pm$ 1.22}   & 53.61 $\pm$ 14.01 &  18.68 $\pm$ 4.06 &  \red{68.40 $\pm$ 1.44} &  16.93 $\pm$ 2.41 & 14.23 $\pm$ 5.05   &   20.69 $\pm$ 5.64   &  18.87 $\pm$ 3.61   & 25.70 $\pm$ 6.78     \\
\midrule

\multirow{3}{*}{PubMed} 
& \emph{clean} &   87.77 $\pm$ 0.14  &   \red{86.94 $\pm$ 0.25}  &   75.55 $\pm$ 0.32  &  84.80 $\pm$ 0.51   &   -   &   84.91 $\pm$ 0.76   &  \blue{89.22 $\pm$ 0.25}   &   76.71 $\pm$ 0.14    &  77.50 $\pm$ 0.54 \\
& SPEIT & \red{85.26 $\pm$ 1.42}  &  \blue{85.13 $\pm$ 0.79}  &  75.07 $\pm$ 0.30  & 84.30 $\pm$ 1.34   &   -   &  39.22 $\pm$ 0.0   & 40.46 $\pm$ 1.69   &  75.58 $\pm$ 1.03    & 77.62 $\pm$ 0.10 \\
& TDGIA &  81.36 $\pm$ 3.09  &  \red{84.88 $\pm$ 0.46}  &  75.78 $\pm$ 0.32  & \blue{83.33 $\pm$ 2.91}   &   -   &  37.96 $\pm$ 1.82   & 44.85 $\pm$ 3.06   &  75.72 $\pm$ 0.70   & 77.32 $\pm$ 0.44 \\

\bottomrule

\end{tabular}}
\end{table}

\subsection{Injection Attacks with Various Attack Strengths}
Under injection attacks with different number of injection nodes and edges, we have performed more experiments with various attack strengths. \cref{tab:inject_diff} shows that the robustness is not significantly affected by the number of nodes and edges injected by attackers as the original is not perturbed in this setting. 

\begin{table}[!hbt]
\centering
\caption{Node classification accuracy (\%) on adversarial examples generated from SPETI and TDGIA under black-box injection setting where different number of injected nodes and edges are applied.} \label{tab:inject_diff} 
\tiny
\makebox[\textwidth][c]{
\begin{tabular}{ccccccccccc} 
\toprule
Dataset & Attack & \# nods/edges injected & Beltrami  & RobustGCN & GNNGuard & GCNSVD  & GAT & GraphSAGE & GIN & APPNP \\
\midrule
\multirow{8}{*}{Cora} 
& SPEIT & 50/50 &   61.87 $\pm$ 0.49  &   36.16 $\pm$ 0.41  &  78.50 $\pm$ 2.27   &   37.50 $\pm$ 0.74   &   38.10 $\pm$ 2.48   &  35.82 $\pm$ 0.01   &  35.82 $\pm$ 0.01   &   36.79 $\pm$ 0.61  \\
& SPEIT & 100/100 &   57.18 $\pm$ 1.68  &   35.82 $\pm$ 0.00  &  79.39 $\pm$ 1.41   &   35.82 $\pm$ 0.00   &   37.22 $\pm$ 1.87   &  35.82 $\pm$ 0.00   &  35.82 $\pm$ 0.00   &  35.82 $\pm$ 0.00  \\
& SPEIT & 150/150 &   55.88 $\pm$ 0.56  &   35.82 $\pm$ 0.00  &  80.78 $\pm$ 0.71   &   35.82 $\pm$ 0.00   &   26.12 $\pm$ 16.29   &  35.82 $\pm$ 0.00   &  35.82 $\pm$ 0.00   &   29.85 $\pm$ 11.94  \\
& SPEIT & 200/200 &   58.21 $\pm$ 2.81  &   35.82 $\pm$ 0.00  &  79.38 $\pm$ 0.64   &   29.88 $\pm$ 11.96   &   29.85 $\pm$ 11.94   &  35.82 $\pm$ 0.00   &  35.82 $\pm$ 0.00   &   22.30 $\pm$ 15.83  \\
& TDGIA & 50/50 &   62.84 $\pm$ 1.17  &   53.28 $\pm$ 8.61  &  78.92 $\pm$ 1.80   &   40.77 $\pm$ 3.34   &   35.64 $\pm$ 12.91   &  39.78 $\pm$ 6.46   &  39.63 $\pm$ 2.38   &   60.52 $\pm$ 4.43  \\
& TDGIA & 100/100 &   62.44 $\pm$ 1.56  &   54.11 $\pm$ 3.18  &  77.99 $\pm$ 2.24   &   42.29 $\pm$ 0.43   &   36.94 $\pm$ 16.96   &  34.24 $\pm$ 9.99   &  37.03 $\pm$ 1.59   &   62.50 $\pm$ 3.47  \\
& TDGIA & 150/150 &   64.18 $\pm$ 0.87  &   52.43 $\pm$ 10.84  &  79.39 $\pm$ 1.74   &   11.94 $\pm$ 0.00   &   39.74 $\pm$ 16.09   &  38.34 $\pm$ 1.48   &  35.82 $\pm$ 0.00   &   61.37 $\pm$ 7.07  \\
& TDGIA & 200/200 &   61.48 $\pm$ 1.68  &   48.88 $\pm$ 6.60  &  80.78 $\pm$ 1.89   &   11.94 $\pm$ 0.00   &   46.74 $\pm$ 12.70   &  38.34 $\pm$ 2.50   &  36.19 $\pm$ 0.75   &   61.10 $\pm$ 0.98  \\
\midrule

\multirow{8}{*}{Citeseer} 
& SPEIT & 50/50 &   65.52 $\pm$ 2.26  &   28.56 $\pm$ 7.87  &  69.72 $\pm$ 1.84   &   21.16 $\pm$ 1.32   &   26.00 $\pm$ 11.14   &  19.75 $\pm$ 1.82   &  23.54 $\pm$ 5.30   &   22.19 $\pm$ 0.86  \\
& SPEIT & 100/100 &   64.74 $\pm$ 1.27  &   17.01 $\pm$ 6.05  &  69.51 $\pm$ 1.61   &   19.75 $\pm$ 2.56   &   19.12 $\pm$ 1.47   &  18.81 $\pm$ 2.74   &  20.61 $\pm$ 2.39   &  20.53  $\pm$ 1.57  \\
& SPEIT & 150/150 &   65.13 $\pm$ 2.27  &   20.69 $\pm$ 1.49  &  70.22 $\pm$ 1.63   &   18.18 $\pm$ 1.81   &   18.97 $\pm$ 1.57   &  18.34 $\pm$ 1.63   &  19.75 $\pm$ 0.00   &   19.75 $\pm$ 0.00  \\
& SPEIT & 200/200 &   64.11 $\pm$ 0.83  &   15.52 $\pm$ 7.19  &  69.36 $\pm$ 1.29   &   18.96 $\pm$ 3.00   &   19.12 $\pm$ 1.26   &   22.10 $\pm$ 5.11   &  21.16 $\pm$ 4.97   &  17.71 $\pm$ 1.37  \\
& TDGIA & 50/50 &   65.77 $\pm$ 1.28  &   38.81 $\pm$ 10.84  &  69.50 $\pm$ 1.86   &   20.77 $\pm$ 2.52   &   19.63 $\pm$ 6.53   &  28.77 $\pm$ 7.73   &  28.65 $\pm$ 5.08   &   54.48 $\pm$ 8.56  \\
& TDGIA & 100/100 &   65.60 $\pm$ 1.68  &   45.61 $\pm$ 6.75  &  66.22 $\pm$ 5.44   &   25.39 $\pm$ 6.33   &   25.08 $\pm$ 10.21   &  35.81 $\pm$ 11.99   &  29.94 $\pm$ 7.26   &   52.74 $\pm$ 6.48  \\
& TDGIA & 150/150 &   66.22 $\pm$ 2.31  &   39.66 $\pm$ 15.80 &  69.36 $\pm$ 0.90   &   24.14 $\pm$ 3.21   &   19.36 $\pm$ 2.14   &  26.57 $\pm$ 9.28   &  27.51 $\pm$ 6.69   &   59.01 $\pm$ 3.93  \\
& TDGIA & 200/200 &   64.66 $\pm$ 2.15  &   38.48 $\pm$ 9.63  &  69.91 $\pm$ 1.54   &   26.49 $\pm$ 4.14   &   22.33 $\pm$ 11.29   &  26.72 $\pm$ 13.41   &  26.73 $\pm$ 7.47   &   51.72 $\pm$ 5.06  \\
\midrule

\multirow{8}{*}{PubMed} 
& SPEIT & 300/300 &   86.66 $\pm$ 0.68  &   75.54 $\pm$ 0.54  &  84.36 $\pm$ 0.58   &   -   &   40.94 $\pm$ 2.47   &  39.22 $\pm$ 0.00   &  76.71 $\pm$ 0.14   &   77.55 $\pm$ 0.54  \\
& SPEIT & 600/600 &   86.80 $\pm$ 0.98  &   75.71 $\pm$ 0.60  &  86.47 $\pm$ 0.58   &   -   &   34.8 $\pm$ 10.99   &  41.10 $\pm$ 1.63   &  76.36 $\pm$ 0.65   &   77.22 $\pm$ 0.43  \\
& SPEIT & 900/900 &   87.10 $\pm$ 0.38  &   76.00 $\pm$ 0.35  &  86.73 $\pm$ 0.48   &   -   &   39.99 $\pm$ 1.55  &  40.78 $\pm$ 1.77   &  76.73 $\pm$ 0.36   &   77.14 $\pm$ 0.70  \\
& SPEIT & 1200/1200 &   87.51 $\pm$ 0.58  &   74.75 $\pm$ 0.90  &  86.19 $\pm$ 0.40   &   -   &   40.82 $\pm$ 1.85   &  40.08 $\pm$ 1.63   &  76.70 $\pm$ 0.72   &   77.17 $\pm$ 0.32  \\
& TDGIA & 300/300 &   85.56 $\pm$ 0.91  &   75.53 $\pm$ 0.36  &  84.00 $\pm$ 1.12   &   -   &   39.78 $\pm$ 0.29   &  60.40 $\pm$ 11.23   &  77.58 $\pm$ 0.71   &   77.45 $\pm$ 0.68  \\
& TDGIA & 600/600 &   86.33 $\pm$ 0.84  &   75.84 $\pm$ 0.23  &  84.21 $\pm$ 0.33   &   -   &   56.00 $\pm$ 9.13   &  70.24 $\pm$ 20.07   &  76.38 $\pm$ 0.75   &   77.37 $\pm$ 0.31  \\
& TDGIA & 900/900 &   87.00 $\pm$ 1.47  &   75.91 $\pm$ 0.55  &  83.87 $\pm$ 1.01   &   -   &   41.44 $\pm$ 10.61   &  52.40 $\pm$ 6.03   &  75.56 $\pm$ 0.58   &   76.88 $\pm$ 0.68  \\
& TDGIA & 1200/1200 &   86.87 $\pm$ 1.41  &   75.62 $\pm$ 0.82  &  84.35 $\pm$ 2.24   &   -   &   57.16 $\pm$ 10.92   &  71.72 $\pm$ 14.02   &  76.49 $\pm$ 0.80   &   77.23 $\pm$ 0.46  \\
\bottomrule

\end{tabular}}
\end{table}

\subsection{Modification Attacks}

In modification attack, attackers can directly flip the original graph's edges and perturb the features of the nodes. We apply the PGD method to randomly flip edges and then perturb node features. In \cref{tab:mod_diff}, we observe that 
\begin{itemize}
    \item[1)] the robustness performance starts to break down when the 60\% of nodes have their features perturbed by $\epsilon=0.1$ (the value of features is in [-1,1]) and 60\% of edges are flipped, and
    \item[2)] as long as feature perturbation is small, the robustness can still be retained even if 80\% of nodes have their features perturbed and 80\% of edges are flipped.
\end{itemize}
In summary, from those extensive experiment results, we observe that our models are more robust against topology perturbation than feature perturbation. 

\begin{table}[!thb]
\centering
\caption{Node classification accuracy (\%) on adversarial examples generated from PGD under black-box \emph{modification} setting where different number of modified nodes and edges are applied. Experiments are conducted on Cora dataset.} \label{tab:mod_diff} 
\small
\begin{tabular}{cccc} 
\toprule
Ratio of nodes/edges modified & Feature perturbation & Beltrami   \\
\midrule
$20\% / 20\%$  &  $\epsilon=0.01$  &  73.73 $\pm$ 0.86  &    \\
$40\% / 40\%$  &  $\epsilon=0.01$  &  73.28 $\pm$ 1.36  &    \\
$60\% / 60\%$  &  $\epsilon=0.01$  &  72.46 $\pm$ 1.48 &    \\
$80\% / 80\%$  &  $\epsilon=0.01$  & 72.31 $\pm$ 2.46 &    \\
$20\% / 20\%$  &  $\epsilon=0.1$  & 61.56 $\pm$ 1.06 &    \\
$40\% / 40\%$  &  $\epsilon=0.1$  & 52.15 $\pm$ 2.31 &    \\
$60\% / 60\%$  &  $\epsilon=0.1$  & 44.03 $\pm$ 0.79 &    \\
$80\% / 80\%$  &  $\epsilon=0.1$  & 40.30 $\pm$ 0.91 &    \\
$80\% / 80\%$  &  $\epsilon=1$  & 39.55 $\pm$ 3.46 &    \\
$80\% / 80\%$  &  $\epsilon=2$  & 37.09 $\pm$ 3.02 &    \\
$80\% / 80\%$  &  $\epsilon=5$  & 37.16 $\pm$ 1.88 &    \\
$80\% / 80\%$  &  $\epsilon=10$  & 37.16 $\pm$ 0.33  &    \\
\bottomrule

\end{tabular}
\end{table}

\subsection{More Discussion about Neural Heat Diffusion}

"Heat" in \cref{tab:flows} of main paper serves as a baseline. Recall that both "Heat" and GRAND/BLEND are derived from heat flow, where GRAND/BLEND uses the attention function to weigh edges while "Heat" uses the constant function, which is not learned from the data and hence not expected to perform well. In a further experiment, we now treat the constant in "Heat" as a trainable variable which is shared by all edges. We denote this new variant as Heat$^+$ and compare it with GAT and APPNP (which were more robust than "Heat") under SPEIT attack in \cref{tab:heat}. We now see that Heat$^+$ is more robust. This constant variable controls the diffusion diffusivity in graph, analogous to the thermal diffusivity in manifold. It is interesting to observe such a phenomenon since the experiment indicates that heat diffusivity also affects the robustness. We believe further investigations in this direction can be performed in future work.

\begin{table}[!hbt]
\tiny
\centering
\caption{Node classification accuracy (\%) on adversarial examples generated from SPEIT.} \label{tab:heat} 
\makebox[\textwidth][c]{
\begin{tabular}{cccccccccccc} 
\toprule
Dataset & GNN & Clean & 10/10  & 15/15 & 20/20 &  25/25 & 30/30 & 35/35 & 40/40 & 45/45 & 50/50 \\
\midrule

\multirow{3}{*}{Citeseer} 
& Heat$^+$ &  69.78 $\pm$ 0.95 &  69.09 $\pm$ 1.61  &   67.77 $\pm$ 1.61   &   66.27 $\pm$ 2.41 & 66.27 $\pm$ 2.81 & 64.26 $\pm$ 1.27 & 61.88 $\pm$ 3.01 & 57.30 $\pm$ 4.07 & 54.61 $\pm$ 5.25 & 54.61 $\pm$ 2.92 \\
& GAT & 69.81 $\pm$ 1.43   &  60.56 $\pm$ 4.77  & 62.01 $\pm$ 4.89  & 45.52 $\pm$ 10.44 & 30.34 $\pm$ 5.40 & 31.16 $\pm$ 6.16 & 31.03 $\pm$ 12.60 & 36.11 $\pm$ 13.85 & 23.13 $\pm$ 3.25 &  26.00 $\pm$ 11.14   \\
& APPNP &  70.75 $\pm$ 0.86   &  68.65 $\pm$ 0.91  & 66.52 $\pm$ 2.90  & 59.87 $\pm$ 3.49 & 56.11 $\pm$ 9.70 & 56.05 $\pm$ 2.37 & 51.22 $\pm$ 8.88 & 38.50 $\pm$ 8.75 & 31.60 $\pm$ 9.80 &  22.19 $\pm$ 0.86  \\
\bottomrule
\end{tabular}}
\end{table}

\section{Further Experiments}\label{sec:supp_more_exp}

We repeat the experiments in \cref{tab:adv} in the main paper but using the three new datasets. The inherent robustness of the proposed graph neural PDEs is again verified by \cref{tab:adv2_rev}, especially on the Flickr and Coauthor datasets.

\begin{table}[!tbh]
\centering
\caption{Node classification accuracy (\%) on adversarial examples generated by SPEIT method. We denote those experiments that are computationally too heavy to run by ``-''.  The best and the second-best result for each criterion are highlighted in \red{red} and \blue{blue}, respectively. } \label{tab:adv2_rev} 
\tiny
\makebox[\textwidth][c]{
\begin{tabular}{cccccccccc} 
\toprule
Dataset & Attack & Beltrami  & RobustGCN & GNNGuard & GCNSVD  & GAT & GraphSAGE & GIN & APPNP \\
\midrule
\multirow{2}{*}{Flickr} 
& \emph{clean} & 49.40  $\pm$  0.12  & 47.66  $\pm$  0.00  & -  & -  & \red{54.45  $\pm$  0.57}  & 53.50  $\pm$  0.02  & 53.57  $\pm$  0.29  & \blue{54.08  $\pm$  0.14}  \\
& SPEIT & \red{49.79  $\pm$  0.68}  & 6.57  $\pm$  0.00  & -  & -   & 6.57  $\pm$  0.00  & \blue{49.71  $\pm$  0.00}  & 49.71  $\pm$  0.00  & 6.57  $\pm$  0.00  \\
& TDGIA & 49.47  $\pm$  0.50  & \blue{52.95  $\pm$  1.58}  & -  & -   & 50.38  $\pm$  0.20  & 50.21 $\pm$ 0.11  & 50.14  $\pm$  0.25  & \red{54.28  $\pm$  0.59}  \\
\midrule

\multirow{2}{*}{Coauthor} 
& \emph{clean} & \red{95.83  $\pm$  0.30}  & 87.75  $\pm$  0.23 & 92.56   $\pm$  0.16  & -   & 92.75  $\pm$  0.15 & \blue{94.53  $\pm$  0.21} & 84.91   $\pm$  0.32  & 87.67  $\pm$  0.16 \\
& SPEIT & \red{94.83   $\pm$  0.12}  & 87.62  $\pm$  0.29 & \blue{92.56   $\pm$  0.16} & -   & 2.59  $\pm$  1.46  & 39.44  $\pm$  13.97 & 39.44  $\pm$  13.97  & 87.66  $\pm$   0.16  \\
& TDGIA & \red{95.06   $\pm$  0.21}  & 87.3  $\pm$  0.29  & -  & -   & 65.32  $\pm$  13.04  & \blue{87.97  $\pm$  3.72}  & 85.12  $\pm$  0.33  & 87.54  $\pm$  0.13  \\
\midrule

\multirow{2}{*}{Amazon Computer} 
& \emph{clean} & 87.86  $\pm$  0.30  & 86.22  $\pm$  0.54  & 88.77  $\pm$  0.05  & 74.79  $\pm$  0.68  & \blue{89.21  $\pm$  0.60}  & \red{89.92  $\pm$  0.33}  & 86.44  $\pm$  0.23  & 82.66  $\pm$  1.54 \\
& SPEIT & 84.90   $\pm$  0.61  & 86.33  $\pm$  0.62 & \red{88.62  $\pm$  0.05}  & 26.79  $\pm$  1.25  & 26.88  $\pm$  16.68  & 29.19  $\pm$  9.35  & \blue{86.44  $\pm$  0.23}  & 82.62  $\pm$  1.55 \\
& TDGIA & 85.50   $\pm$  0.15  & \red{86.69  $\pm$  0.51}  & -  & -   & 55.45  $\pm$  23.07  & 63.14  $\pm$  10.59  & \blue{86.65  $\pm$  0.57}  & 83.44  $\pm$  1.59  \\
\bottomrule
\end{tabular}}
\end{table}

\section{Further Ablation Studies}
\label{sec:supp_ablation}

We repeat the experiments in \cref{tab:flows} of the main paper but using three new datasets. Similar to what we have stated in the main paper, we observe the advantage of the proposed mean curvature flow and Beltrami flow in terms of robustness against adversarial attacks.  

\begin{table}[!tbh]
\centering
\caption{Node classification accuracy (\%) on adversarial examples using graph neural PDEs induced from different flows, where implicit Adam PDE solver with step size 2 is used.  } \label{tab:flows2}
\small
\begin{tabular}{cccccc} 
\toprule
Dataset & Attack & Beltrami & Mean Curvature & GRAND/BLEND & Heat \\
\midrule
\multirow{3}{*}{Coauthor} 
& \emph{clean} & \red{95.81  $\pm$ 0.38} &  \blue{95.66  $\pm$ 0.18}   &  94.35  $\pm$ 0.33   &   92.87  $\pm$  0.55 \\
& SPEIT        & \blue{95.10  $\pm$ 0.35} &  \red{95.41  $\pm$ 0.28}   &  66.63  $\pm$ 9.76   &   34.18  $\pm$  8.85  \\
& TDGIA        & \blue{95.06  $\pm$ 0.21} &  \red{95.41  $\pm$ 0.17}   &  85.97  $\pm$ 4.50   &   62.39  $\pm$  15.96   \\
\midrule

\multirow{3}{*}{Amazon Computer} 
& \emph{clean} & 87.86  $\pm$ 0.30  &  87.59  $\pm$ 0.44   &  \blue{90.59  $\pm$  0.35}  &  \red{90.59   $\pm$ 0.57}  \\
& SPEIT        & \blue{84.90  $\pm$ 0.61}  &  \red{84.55  $\pm$ 1.03}   &  75.91  $\pm$  12.22  &  46.82   $\pm$ 14.42   \\
& TDGIA        & \blue{85.50  $\pm$ 0.15}  &  85.34  $\pm$ 0.78   &  \red{86.31  $\pm$  2.50}  &  76.60   $\pm$  10.48   \\
\midrule

\multirow{3}{*}{Flickr} 
& \emph{clean} & \red{49.42  $\pm$ 0.12} &  47.66  $\pm$  1.51  &  \blue{48.52  $\pm$  0.19}  &   46.65  $\pm$ 1.30  \\
& SPEIT        & \red{49.76  $\pm$ 0.63} &  49.74  $\pm$  1.18  &  \blue{49.75  $\pm$  0.06}  &   49.70  $\pm$ 0.02   \\
& TDGIA        & \red{49.47  $\pm$ 0.50} &  47.39  $\pm$  1.43  &  \blue{48.63  $\pm$  0.09}  &   47.05  $\pm$ 1.20    \\

\bottomrule
\end{tabular}
\end{table}

\section{Extension}\label{sec:supp_extension} 
The graph Laplacian and curvature can be generalized to an operator that can be thought of as the discrete analogue of the $p$-Laplacian in the continuous case \cite{zhou2005regularization,dibenedetto2011harnack}: 
\begin{align}\label{eqn:pLap}
    \frac{\partial \varphi(u,t)}{\partial t}=\frac{1}{2}{\rm div}\left(\|\nabla \varphi\|^{p-2}\nabla \varphi\right)(u,t),
\end{align}
where $p=1,2$ correspond to the mean curvature and heat (GRAND/BLEND) equations, respectively. Here, we consider the cases where $p>2$. Note that when $p>2$, as opposed to the cases where $p\leq 1$, e.g., $p=1$ for mean curvature and Beltrami flows, the edges that connect nodes with similar features are potentially preserved in the diffusion process. \cref{tab:pLap} shows that the PDE's robustness decreases as $p$ increases.

\begin{table}[!tbh]
\centering
\caption{Node classification accuracy (\%) on adversarial examples using graph neural PDEs induced from $p$-Laplacian flow where $p=3,4$, where implicit Adam PDE solver with step size 2 is used. All experiments are done on Citeseer dataset. } \label{tab:pLap}
\small
\begin{tabular}{ccccc} 
\toprule
Attack & $p$-Laplacian $(p=3)$ & $p$-Laplacian $(p=4)$ & Mean Curvature & Beltrami \\
\midrule
\emph{clean} & 69.15 $\pm$ 1.37  &  67.27 $\pm$ 1.14  & \red{70.50 $\pm$ 1.63}  &  \blue{70.41 $\pm$ 1.38} \\
SPEIT        & 63.76 $\pm$ 0.85  &  62.44 $\pm$ 1.14  & \blue{65.39 $\pm$ 1.62}  &  \red{65.52 $\pm$ 2.26}  \\
TDGIA        & 64.89 $\pm$ 0.77  &  63.39 $\pm$ 2.00  & \red{66.83 $\pm$ 1.59}  &  \blue{65.77 $\pm$ 1.28}   \\

\bottomrule
\end{tabular}
\end{table}

\section{Implementation Details}\label{sec:supp_imp_det}
By default, all the models are implemented using three layers, with layer normalization, 50\% dropout at the end of each layer, and hidden feature dimensions $64-64$. Here are some notes that should be taken:
\begin{itemize}
    \item When dealing with the Cora dataset, PDEs induced by Beltrami and mean curvature flows tend to underfit. Thus, we implement them using four layers with hidden feature dimensions $128-128-128$.
    
    \item For all neural PDEs, node features are diffused independently over layers, i.e., each layer solves its own PDE. At each layer, once a diffusion process is over, the parameters in the associated PDE are reset. This operation is effective to alleviate overfitting problem.
    \item By default, the integral period for all PDEs is set to $[0,1]$.
\end{itemize}
All experiments are repeated 5 to 10 times with different random seeds.

For all the attacks, the maximum number of nodes and edges that are allowed to be perturbed are summarized in \cref{tab:attack_max_min}.
\begin{table}[!tbh]
\centering
\caption{Statistic of attacks' budgets on each dataset} \label{tab:attack_max_min} 
\small
\begin{tabular}{ccc} 
\toprule
Dataset & max \# Nodes & max \# Edges \\
\midrule

 Cora &  50  & 50   \\
\midrule

 Citeseer &  50  & 50   \\
\midrule

 PubMed &  300  & 300   \\
\midrule

  Coauthor &  150  & 300   \\
 \midrule
 
  Amazon Computer &  100  & 200 \\
 \midrule

 Flickr &  1000  & 5000 \\

\bottomrule
\end{tabular}
\end{table}

GNNGuard implemented in \cref{tab:adv} of the main paper is based on GCN, where each layer in GNNGuard contains two operations: 1) adjusting attention coefficients by pruning likely fake edges and assigning less weight to suspicious edges; and 2) a normal GCN layer. Regarding BeltramiGuard implemented in \cref{tab:BelGuard} of the main paper, we replace the second component as mentioned above, i.e., a GCN layer, with a PDE layer induced by Beltrami flow. 

\section{Proofs of Results}\label{sec:supp_proof}

In this section, we provide detailed proofs of the results stated in the main paper. The time-variant analogy of \cite{zheng1999stability} as discussed after \cref{eq:heat_cauchy2} in the main paper is also presented.

\subsection{Proof of \cref{lem:linear_graph}}
\label{sect:proof_thm:reg_simplex}

For the system $ \frac{\partial \varphi(u,t)}{\partial t}= - \Delta \varphi(u,t)$, the solution is given by  $\varphi(u,t)=e^{-t\Delta}\varphi(u,0)$. We therefore need to derive a bound for  $\vertiii{e^{-t\Delta}-e^{-t\tilde{\Delta}}}$. Note for matrix exponent, in general $e^{X+Y}\ne e^{X} e^{Y}$ unless $X$ and $Y$ commute (i.e. $XY=YX$). However, from \cite[eq (3.5) ]{van1977sensitivity}, we have 
\begin{align*}
    \vertiii*{e^{-t\Delta}-e^{-t\tilde{\Delta}}}
    &\le t\vertiii*{\Delta - \tilde{\Delta}}\vertiii*{e^{-t\Delta}}e^{\vertiii{t(\Delta - \tilde{\Delta})}},
\end{align*}
where the positive definiteness of $\Delta$ is used.
We further analyze $\Delta - \tilde{\Delta}$. Since $\tilde{\bW}=\bW+ \bE$ with $\vertiii{\bE}=\varepsilon$, we have $\vertiii{\bD - \tilde{\bD}}= O(\varepsilon)$ because norms for a finite dimensional space are equivalent \cite[Theorem 5.4.4.]{horn2012matrix}. Let $\bE'=  \tilde{\bD}^{-1/2} - \bD^{-1/2}$. We also have $\vertiii{\bE'}=O(\varepsilon)$.

It follows that
\begin{align*}
        \vertiii*{\Delta - \tilde{\Delta}} 
        & = \vertiii*{\bD^{-1/2}\bW\bD^{-1/2} - \tilde{\bD}^{-1/2} \tilde{\bW} \tilde{\bD}^{-1/2}}\\
        & =  \vertiii*{\bD^{-1/2}\bW\bD^{-1/2} - (\bD^{-1/2}+\bE')(\bW+\bE) (\bD^{-1/2}+\bE')}\\
        & = \vertiii*{\bE'(\bW+\bE')(\bD^{-1/2}+\bE')+\bD^{-1/2}\bE(\bD^{-1/2}+\bE')+\bD^{-1/2}\bW\bE' }\\
        & = O(\varepsilon).
\end{align*}
We therefore obtain the conclusion that $\vertiii{e^{-t\Delta}-e^{-t\tilde{\Delta}}}=O(\varepsilon te^{-\rho t})$ for some constant $\rho>0$ since $\vertiii{e^{-t\Delta}}=O(e^{-t\rho})$ according to \cite{van1977sensitivity} and the fact that $\Delta$ is positive definite according to our assumption. The conclusion $\|\varphi(u,t)-\tilde{\varphi}(u,t)\|= O(\varepsilon te^{-\rho t})$ then follows.

\subsection{Proof of \cref{lem:lyap_stable}}
\begin{proof}
Since $\bA$ is right stochastic, all eigenvalues of $\bA - \bI$ have non-positive real parts with module $\le 1$, the stability of \cref{eqn:mtx_heat} is ensured. Since $\bPsi^{-1}\bA\odot \bB$ is right stochastic, all eigenvalues of $\bPsi^{-1}\bA\odot \bB - \bI$ have non-positive parts with module $\le 1$ and \cref{eqn:mtx_beltrami} are stable.
\end{proof}

\subsection{The Time-Variant Case}
In this section, we provide an analysis of the time-variant system as alluded to before \cref{lem:linear_graph}. 
In the time-variant system, we have
\begin{align}
    \frac{\partial \varphi(u,t)}{\partial t}= -\Delta(t) \varphi(u,t), \label{eq:difad}
\end{align}
where $\Delta(t)$ is a time-variant Laplacian operator.
Let the state transition matrix be $\Phi_{\Delta}\left(t ; 0\right)$ so that the solution of \cref{eq:difad} is given by $\varphi(u, t)=\Phi_{\Delta}\left(t ; 0\right)\varphi(u, 0)$ \cite{chen1999linear}.

In \cite{hinrichsen1989robustness}, the authors discuss the stability radius of time-variant systems. The system still preserves stability if the perturbation of $\Delta(t)$ is smaller than a stability radius.
In this discussion, we assume the perturbation is small, i.e., it is inside the stability radius, so that both $\Delta(t)$ and its perturbed version $\tilde{\Delta}(t)$  generate exponentially stable solutions. In other words, for a given matrix norm $\vertiii{}$, there exist positive constants $M$, $\tilde{M}$, $\omega$ and $\tilde{\omega}$ such that  
\begin{align}
\vertiii*{\Phi_{\Delta}(t, s)} \le M e^{-\omega(t-s)}, \text{ and } \vertiii*{\Phi_{\tilde{\Delta}}(t, s)} \le \tilde{M} e^{-\tilde{\omega}(t-s)}, \quad t \ge s \ge 0. \label{eq:assump1}
\end{align}
To simplify the analysis, we further assume the perturbation is small enough such that the exponent difference is also small, i.e.,
\begin{align}
   | \omega - \tilde{\omega}| < \omega.\label{eq:assump2}
\end{align}
We now consider the solution difference after perturbation. 
Let $\bD(t)$ and $\bW(t)$ be the time-variant version of $\bD$ and $\bW$, respectively. They are assumed to satisfy \cref{eq:assump1} and \cref{eq:assump2}.

\begin{Lemma}\label{lem:linear_graph_timevar}
 Let $\Delta(t) = \bD^{-1/2}(t)(\bD(t) - \bW(t))\bD^{-1/2}(t)$ and $\tilde{\Delta}(t)=\tilde{\bD}^{-1/2}(t)(\tilde{\bD}(t) - \tilde{\bW}(t))\tilde{\bD}^{-1/2}(t)$ with $\tilde{\bD}(t)$ being the diagonal degree matrix for $\tilde{\bW}(t)=\bW(t)+ \bE(t)$, where $\set{\tilde{\bW}(t) \given t \geq0}$ satisfies \cref{eq:assump1} and \cref{eq:assump2}. Denote $\varepsilon(t)=\vertiii{\bE(t)}$.
 We have $\|\varphi(u,t)-\tilde{\varphi}(u,t)\|= O\parens*{e^{-\rho t} \int_0^{t}\varepsilon(\tau)\ud \tau }$ for some constant $\rho>0$.
\end{Lemma}
\begin{proof}
 For the perturbed system, we have
 \begin{align}
     \frac{\partial \varphi(u,t)}{\partial t}& = - \tilde{\Delta}(t) \varphi(u,t) \nn
     & = - \Delta(t) \varphi(u,t) + \left(\Delta(t) -\tilde{\Delta}(t) \right)\varphi(u,t). \label{eq:tefdad}
 \end{align}
Formally, \cref{eq:tefdad} can be interpreted as a closed loop system obtained by applying the dynamical feedback $\left(\Delta(t) -\tilde{\Delta}(t) \right)\varphi(u,t)$. According to \cite{chen1999linear}, the solution of the above system, denoted by $\tilde{\varphi}(u,t)$, satisfies
\begin{align*}
    \tilde{\varphi}(u,t) & = \Phi_{\Delta}\left(t, 0\right) \varphi(u,0) + \int_{0}^{t} \Phi_{\Delta}(t, \tau)\left(\Delta(\tau) -\tilde{\Delta}(\tau)\right) \tilde{\varphi}(u,\tau)  \ud \tau \nn
    & = \varphi(u, t) + \int_{0}^{t} \Phi_{\Delta}(t, \tau)\left(\Delta(\tau) -\tilde{\Delta}(\tau)\right) \tilde{\varphi}(u,\tau)  \ud \tau.
\end{align*}
Similar to the proof in \cref{sect:proof_thm:reg_simplex} for the time-invariant case, we have $\vertiii{\Delta(\tau) -\tilde{\Delta}(\tau)} = O(\varepsilon(t))$. It follows that 
\begin{align*}
     \|\tilde{\varphi}(u,t) - \varphi(u, t)\| & = \norm*{\int_{0}^{t} \Phi_{\Delta}(t, \tau)\left(\Delta(\tau) -\tilde{\Delta}(\tau)\right) \tilde{\varphi}(u,\tau)  \ud \tau} \nn
    & =  \norm*{\int_{0}^{t} \Phi_{\Delta}(t, \tau)\left(\Delta(\tau) -\tilde{\Delta}(\tau)\right) \Phi_{\tilde{\Delta}}(\tau, 0) \varphi(u,0) \ud \tau} \nn
       & \le \int_{0}^{t} \vertiii*{\Phi_{\Delta}(t, \tau)}\vertiii*{\Delta(\tau) -\tilde{\Delta}(\tau)} \vertiii*{\Phi_{\tilde{\Delta}}(\tau, 0)} \ud \tau\cdot \|\varphi(u,0) \|   \nn
     & \le \int_{0}^{t} M e^{-\omega(t-\tau)} \tilde{M} e^{-\tilde{\omega}(\tau-0)} \varepsilon(\tau) \ud\tau \cdot \|\varphi(u,0) \|   \nn
     & =  O\parens*{e^{-\rho t} \int_0^{t}\varepsilon(\tau)\ud\tau}
\end{align*}
for some constant $\rho>0$. In the second equality we have used the fact that the initial point for both the unperturbed and perturbed system is $ \varphi(u,0) $, and the last equality follows from \cref{eq:assump1} and \cref{eq:assump2}.
\end{proof}

\begin{Remark_A}
For the sake of simplicity, in this work, we only present theoretical analysis for the time-invariant or time-variant graph Laplacian in \cref{lem:linear_graph} and \cref{lem:linear_graph_timevar} for heat diffusion \cref{eq:heatgraph}. However the more complicated case where the graph Laplacian $\Delta$ depends on node features $u$ as implemented in \cref{eqn:w} is left for future work. From experiment results, for example, the results shown in \cref{tab:flows} in the main paper and \cref{tab:flows2} in this supplementary material, the time-invariant case already preserve some robustness as compared to non-PDE GNNs in \cref{tab:adv} and \cref{tab:adv2_rev}, which validates our theoretical analysis. The more robust mean curvature flow and Beltrami flow proposed in the paper are highly non-linear, making a theoretical analysis difficult. However, our analysis for the time-variant case provides some insights as to why robustness is present in these cases, as demonstrated in our experiments.

\end{Remark_A}

\section*{Broader Impact}
Our work develops robust GNNs to mitigate the threat of adversarial attacks, which can lead to reliable deployment of automation in various applications like in sensor networks, transportation networks, and manufacturing. This may potentially lead to the replacement of repetitive tasks or jobs that are traditionally performed by humans. However, automation and artificial intelligence (AI) can lead to better productivity, efficiency and cost-effectiveness with an overall increase in societal living standards. By incorporating the ability to defend against adversarial attacks, our research can lead to more secure and robust adoption of AI technologies. However, potential failures and engineering issues remain challenging open problems.

\section*{Acknowledgement}
This research is supported by the Singapore Ministry of Education Academic Research Fund Tier 2 grant MOE-T2EP20220-0002 and A*STAR under its RIE2020 Advanced Manufacturing and Engineering (AME) Industry Alignment Fund – Pre Positioning (IAF-PP) (Grant No. A19D6a0053).

\bibliographystyle{IEEEtran}
\bibliography{IEEEabrv,StringDefinitions,adv_dnn}

\end{document}